\documentclass[sigconf]{aamas}  

\AtBeginDocument{%
  \providecommand\BibTeX{{%
    \normalfont B\kern-0.5em{\scshape i\kern-0.25em b}\kern-0.8em\TeX}}}
    
\usepackage{booktabs}    
\usepackage{flushend} 

\setcopyright{ifaamas}  
\copyrightyear{2020} 
\acmYear{2020} 
\acmDOI{} 
\acmPrice{} 
\acmISBN{} 
\acmConference[AAMAS'20]{Proc.\@ of the 19th International Conference on Autonomous Agents and Multiagent Systems (AAMAS 2020)}{May 9--13, 2020}{Auckland, New Zealand}{B.~An, N.~Yorke-Smith, A.~El~Fallah~Seghrouchni, G.~Sukthankar (eds.)}  

\usepackage{booktabs}
\usepackage{hyperref}
\usepackage{url}
\usepackage[utf8]{inputenc} 
\usepackage[T1]{fontenc}    
\usepackage{booktabs}       
\usepackage{amsfonts}       
\usepackage{nicefrac}       
\usepackage{microtype}      
\usepackage{mathtools}  
\usepackage{tabulary}
\usepackage{booktabs}
\usepackage{wrapfig,lipsum,booktabs}
\usepackage{mathrsfs}
\usepackage{bbm}

\newcommand{\yvar}{y}
\newcommand{\numy}{{n_y}}

\newcommand{\actionset}{\mathcal{A}}
\newcommand{\stateset}{\mathcal{S}}
\newcommand{\yset}{Y}

\begin{document}

\title{Maximizing Information Gain in Partially Observable Environments via Prediction Rewards}  



\author{Yash Satsangi}
\affiliation{%
  \institution{University of Alberta}
}
\email{ysatsang@ualberta.ca}

\author{Sungsu Lim}
\affiliation{%
  \institution{University of Alberta}
}
\email{sungsu@ualberta.ca}

\author{Shimon Whiteson}
\affiliation{%
  \institution{University of Oxford}
}
\email{shimon.whiteson@cs.ox.ac.uk}

\author{Frans A.\ Oliehoek}
\affiliation{%
  \institution{Technical University Delft}
}
\email{f.a.oliehoek@tudelft.nl}

\author{Martha White}
\affiliation{%
  \institution{University of Alberta}
}
\email{whitem@ualberta.ca}

\begin{abstract}
Information gathering in a partially observable environment can be formulated as a reinforcement learning (RL), problem where the reward depends on the agent's uncertainty. For example, the reward can be the negative entropy of the agent's belief over an unknown (or hidden) variable. Typically, the rewards of an RL agent are defined as a function of the state-action pairs and not as a function of the belief of the agent; this hinders the direct application of deep RL methods for such tasks. This paper tackles the challenge of using belief-based rewards for a deep RL agent, by offering a simple insight that maximizing any convex function of the belief of the agent can be approximated by instead maximizing a prediction reward: a reward based on prediction accuracy. In particular, we derive the exact error between negative entropy and the expected prediction reward. This insight provides theoretical motivation for several fields using prediction rewards---namely visual attention, question answering systems, and intrinsic motivation---and highlights their connection to the usually distinct fields of active perception, active sensing, and sensor placement. Based on this insight we present deep anticipatory networks (DANs), which enables an agent to take actions to reduce its uncertainty without performing explicit belief inference. We present two applications of DANs: building a sensor selection system for tracking people in a shopping mall and learning discrete models of attention on fashion MNIST and MNIST digit classification. 
\end{abstract}

%


\maketitle
\section{Introduction}

To act intelligently, an agent must be able to reason about its uncertainty over certain variables in its environment. \emph{Active perception} \citep{Bajcsy88,Bajcsy16} is the ability of an agent to reason about its uncertainty and take actions to reduce it. The aim of the agent is to take actions, to collect observations, that help it predict the value of an unknown\footnote{We use the term unknown variable instead of hidden variable, because we assume that we have access to this unknown variable during training, as is standard in supervised learning. A hidden variable, on the other hand, is never available.} variable, say $\yvar$ at each time step $t$. For example, consider the sensor selection task \citep{Hero11, SatsangiJournal17}, where an agent has access to a set of available sensors to infer the unknown position of a person in a shopping mall ($\yvar$). At each time step $t$, due to resource constraints, the agent must select a subset of the sensors from which to collect the observations. Another example is the visual attention task \citep{Mnih14}, where an agent must sequentially attend to parts of an image to determine if an object is present  ($\yvar = 1$ or 0). 

The problem of taking informative actions---or selecting informative observations---to minimize (future) uncertainty can be formulated as a reinforcement learning problem. The agent takes actions and receives rewards for reducing uncertainty. The key question is how to compute such rewards. The most straightforward approach is as follows. 
At each time step, the agent maintains a probability distribution over the unknown variable $\yvar$. 
The agent takes actions $a^t$ to collect observations (denoted by $z$) about this unknown variable. The agent can then update its probability distribution over the unknown variable $p^{t+1}(\yvar) = \Pr(\yvar| z^1, z^2, \dots, z^{t+1}, a^0, a^1, \dots, a^{t})$. 
The reward corresponds to expected reduction in uncertainty, after taking an action. 
A common definition for reduction in uncertainty is the expected \textit{information gain} \citep{Krause05uai}: $\mathbb{E}_{\Pr(z^{t+1}|p^{t},a)}[H(p^t) - H(p^{t+1})]$, where $H(p^t) = - \sum_{\yvar \in Y} (p^t(\yvar) \log (p^t(\yvar))$ is the entropy of the the probability distribution $p^t$. The expectation is over the possible observations $z^{t+1}$ if the agent takes action $a$. 

Unfortunately, computing these rewards can be prohibitively expensive. 
Given a model of the world---the conditional probability distributions $\Pr(z^{t+1}| \yvar^{0:t+1}, a^{0:t})$ and $\Pr(\yvar^{t+1}|\yvar^{0:t},a^{0:t})$---the agent can perform explicit belief inference to exactly compute the information gain of taking an action and so compute the action that maximizes it \citep{Krause05uai, SatsangiJournal17}. Such models must be either manually specified, or learned if a dataset is available, which requires substantial expert knowledge and significant human effort. Even when a model of the world is available, performing explicit belief inference can be expensive or even intractable.
In such cases, approximate belief inference methods such as particle filters \citep{Doucet09} or variational approximation \cite{Igl18} must be used to compute the information gain.

In this paper we present a simple model-free reinforcement learning approach that allows an agent to take actions that maximize its information gain \emph{without} performing explicit belief inference. We start by presenting a simple insight that shows that any convex function of the belief of an agent (about an unknown variable) can be approximated simply by using prediction rewards, for example, +1 for a correct prediction and 0 for an incorrect prediction. 
Given an arbitrary prediction reward, we establish the exact error bounds the agent would incur for acting greedily with respect to the given prediction reward in comparison to actions that maximize the information gain of the agent. We show that in principle the prediction rewards can be designed to optimize this error.

The practice of providing an agent with prediction rewards is common in sub-fields such as visual attention \citep{Mnih14}, question answering systems \cite{Nara16} and intrinsic motivation \citep{Pathak17}; this work provides theoretical motivation for these strategies and further generalizes the types of rewards and prediction problems that can be considered. Furthermore, the framework put forth unifies disparate areas that are in fact working on similar approaches, namely the fields already using prediction rewards and fields where it is common to maximize information gain, including active perception \citep{SatsangiJournal17}, active sensing \citep{Kreucher05} and sensor placement \citep{Krause05uai}. 

We use the our theoretical result to develop \emph{deep anticipatory networks} (DANs) as a principled framework to leverage the power of deep RL to minimize uncertainty without performing explicit belief inference. A DAN consists of two neural networks: a $\mathsf{Q}$ network that selects sensory actions and a model, $\mathsf{M}$ network that predicts the state of the world based on the observations generated by those sensory actions. The main idea behind DAN is to train the $\mathsf{Q}$ network and $\mathsf{M}$ network simultaneously:  the $\mathsf{Q}$ network learns a $Q$-function that estimates how much each sensory action would help the $\mathsf{M}$ network to predict the current state. Given some ground truth data, the $\mathsf{M}$ network learns to predict the current state in a supervised way, given the observations generated by the sensory actions that were selected according to the $Q$-network. 

Finally, we empirically test our algorithm in two settings: sensor selection and attention. 
We build a sensor selection system for tracking people that scales to a large number of people. 
Using DAN we learn a policy for sensor selection and we show its performance on test data (when deployed) in comparison to other baselines that reward the agent using a heuristic that is based on the coverage of the sensor.  We also apply DAN to a visual attention task where an agent must predict an MNIST class given only a partial observation of it. Our experiments on the MNIST \citep{Lecun98} and fashion MNIST \cite{Xiao17} datasets show that formulating the visual attention tasks as a continual problem where the agent is rewarded throughout the episode is superior to the terminal reward formulation common in the literature. 

\section{Problem Setting}
We model the world as a partially observable Markov decision process (POMDP) \citep{Kaelbling98} with finite state, action and observation space. At each time step $t$, the environment is in hidden state $s \in \stateset$, the agent takes an action $a \in \actionset$ and the environment transitions to a new state $s' \in \stateset$. Additionally, the agent receives an observation $z \in \Omega$ that is correlated with a target variable $\yvar \in \yset = \{1, 2, \dots, \numy\}$ that is a function of $s$, $\yvar = I(s)$. 

The aim of the agent is to predict the target correctly on each step. At each time step, the agent maintains a probability distribution over $\yvar$ given the previous actions and observations, \\ $\Pr(\yvar | z^{t}, z^{t-1} \dots z^{1}, a^{t-1}, a^{t-2} \dots a^{0} )$. After taking action $a^{t}$ and receiving observation $z^{t+1}$, the agent can update the probability distribution $\Pr(\yvar | z^{t+1}, z^{t}, \dots, z^{1}, a^{t}, a^{t-1}, \dots a^{0})$ using the Bayes rule. This has been formalized as a $\rho$POMDP \cite{Araya10} where the reward is defined as the negative entropy of the probability distribution over $\yvar$. This formulation, however, requires access to the true probability distributions of the POMDP. Instead, we only assume access to a labelled dataset for training, where for a sequence of observations we are given the corresponding targets. 
For a sensor selection task, such a dataset can be obtained by investing a one-time effort to collect and label sets of observations, without inferring or knowing anything about hidden states or the underlying probabilities.

\section{A connection between Information gain and prediction rewards}
In this section we provide a bound between the negative entropy and prediction rewards, which correspond to rewarding the agent for correct predictions of the target variable. In particular, we show that prediction rewards provide a set of tangents that form a lower-bound to the negative entropy. We discuss at the end of the section how this implies that maximizing expected prediction rewards---as is done by a reinforcement learning agent---provides an effective proxy to maximizing expected information gain. We first provide an informal theorem statement, and then introduce the required notation to prove the main results. 

Let $\mathbf{b} = (b_1, b_2, \dots b_{\numy})$ denote a probability vector in an $\numy$ dimensional vector space such that $\sum_{i \in \{1, 2, \dots \numy \}} b_{i} = 1$ ($Y = \{1, 2, \dots, \numy\}$), and let $H(\mathbf{b})$ be the Shannon entropy defined by $H(\mathbf{b}) = - \sum_{i \in Y} b_i \log b_i$. The vector $\mathbf{b}$ corresponds to the agents prediction about the what target variable is most probable, given the history of observations. The goal of the agent is to select actions to maximize information gain, and so decrease the entropy of the probabilities $\mathbf{b}$: maximize the negative entropy. We can instead consider maximizing an expected 0-1 prediction reward for the most probable class, $\max_i b_i$. 

\vspace{0.1cm} 
\noindent
\textbf{Informal Theorem Statement:}
The difference between the negative entropy $-H(\mathbf{b})$ and the expected 0-1 prediction reward $\max_i b_i$ (shifted by the a constant that is the same on every step) is upper bounded by $ -1 + \log(e + \numy -1)$. 

\subsection{Main Theoretical Result}

Let $\rho(\mathbf{b})$ be any convex function of the probabilities $\mathbf{b}$, such as $\rho(\mathbf{b}) = -H(\mathbf{b})$.
The equation of a tangent plane to $\rho$  is given by: $\langle \mathbf{b}, \triangledown \rho(\mathbf{b_0}) \rangle + c_{\mathbf{b}_0}$, where $c_\mathbf{b_0}$ is a constant and $\triangledown \rho(\mathbf{b}_0)$ is the gradient of $\rho$. Though generically complex to compute, $c_{\mathbf{b}_0}$ can be computed analytically for certain functions, using Fenchel conjugates (see \citet*{boyd04} for a comprehensive introduction).  
Here, we describe the two most relevant properties for this paper: 

\begin{table}
\caption{Summary of notation}\label{wrap-tab:1}
\begin{tabular}{c c}\toprule  
$\hat{\yvar}$ & a random variable that denotes a prediction  \\ \midrule
\shortstack{$h^{t}$ \\ \newline} & \shortstack{ the action ($a$)-observation($z$) history \\ $h^{t} = \langle a^0, z^1, a^1, \dots, a^{t-1}, z^{t} \rangle$}. \\ \midrule
$\mathbf{b} $  & denotes a probability vector \\  \midrule
$\rho(\mathbf{b})$  & a convex and differentiable function of $\mathbf{b}$ \\    \midrule
$\rho^{*}(\mathbf{b})$ & the Fenchel conjugate of $\rho(\mathbf{b})$ \\ \midrule
$\triangledown \rho(\mathbf{b})$ &   the gradient of $\rho(\mathbf{b})$  \\ \midrule
$\triangledown \rho(\mathbf{b})_i$ &  the $i^{th}$ entry in the vector $\triangledown \rho(\mathbf{b})$ \\ \midrule
$R(\yvar,\hat{\yvar})$ & the prediction reward function \\ \midrule
\shortstack{$\mathbf{r}_{j}$ \\ \newline} & \shortstack{a reward vector, each entry $r_i$ of $\mathbf{r}_j$ is the \\ scalar reward agent gets for $\hat{\yvar} = j$ when true $\yvar = i$.} \\ \midrule
$\log$ & natural logarithm \\ \midrule
\end{tabular}
\end{table}

\textbf{Property 1}: If $\rho(\mathbf{b})$ is convex, closed and differentiable, then $c_{\mathbf{b}_0}$ is the negative of the \emph{Fenchel conjugate}  of $\rho(\mathbf{b})$ at $\triangledown \rho(\mathbf{b}_0)$, that is, $c_{\mathbf{b}_0} = - \rho^*(\triangledown \rho(\mathbf{b}_0))$, where $\rho^{*}$ denotes Fenchel conjugate of $\rho$ \citep{bauschke12,boyd04}. 

\textbf{Property 2}: The Fenchel conjugate of the negative entropy is the log-sum-exp function, $\log(\sum_{i} e^{x_i})$ \cite[Page 93]{boyd04}.  

Property 1 and 2 give that for $\rho(\mathbf{b}) = -H(\mathbf{b})$,  the constant term is $c_{\mathbf{b}_0} = - \log(\sum_{i=1}^{n} e^{\triangledown \rho(\mathbf{b}_0)_i})$, where $\triangledown \rho(\mathbf{b}_0)_i$ denotes the $i^{th}$ entry in the vector $\triangledown \rho(\mathbf{b}_0)$.
Now, let $\hat{\yvar} \in Y = \{1,2,3 \dots, \numy\}$ denote a prediction that is input to a reward function $R(\yvar,\hat{\yvar})$, which gives a scalar value $r_{i, j}$  for each combination of $i,j \in Y $. 
Let $R(\yvar,\hat{\yvar}=j)$, the reward vector associated with predicting $\yvar$ as $j$ using $\hat{\yvar}$ be denoted by the vector $\mathbf{r}_{j}$. That is, each entry $r_i$ in $\mathbf{r}_j$ is the reward for predicting $\hat{\yvar}$ as $j$ when the true value of $\yvar$ is $i$.
Given a probability vector $\mathbf{b}$, the expected reward for assigning $\hat{\yvar}=j$ is 
\begin{equation}
\rho'(\mathbf{b},\hat{\yvar}=j) = \langle \mathbf{b} , \mathbf{r}_j \rangle = \sum_{i \in Y} b_i r_{i,j},
\end{equation}
which leads to the following lemma.

\begin{lemma} \label{eq:tangent}
If $\rho$ is a closed, convex and differentiable function of $\mathbf{b}$ and $\mathbf{r}_j$ is in the set of all possible values of the gradients of $\rho$ then $\rho'(\mathbf{b},j) - \rho^*(\mathbf{r}_j) =  \langle \mathbf{b} , \mathbf{r}_j \rangle - \rho^*(\mathbf{r}_j)$ is a tangent to the curve $\rho(\mathbf{b})$ at $\mathbf{b_0}$ that satisfies $\triangledown \rho(\mathbf{b_0}) = \mathbf{r}_j$ for any fixed $j \in Y$. 
\end{lemma}
\begin{proof}
Property 1 imply that the equation of a tangent to the curve $\rho(\mathbf{b})$ is $\langle \mathbf{b} , \triangledown \rho(\mathbf{b_0}) \rangle - \rho^*(\triangledown \rho(\mathbf{b_0}))$. If $\mathbf{r}_{j} = \triangledown \rho(\mathbf{b_0})$ then $\langle \mathbf{b} , \mathbf{r}_j \rangle - \rho^*(\mathbf{r}_j)$ is a tangent to the curve $\rho(\mathbf{b})$. The condition that $\mathbf{r}_j$ is in the set of all possible values of gradients of $\rho$ is required for $\rho^*$ to be defined (and for $\triangledown \rho(\mathbf{b_0}) = \mathbf{r}_j$ to have a solution).
\end{proof}

We can use this lemma to show that the maximum over these tangent planes forms a lower bound on $\rho(\mathbf{b})$. When $\rho$ is the negative entropy, this maximum over tangent planes precisely corresponds to the expected prediction reward, shifted by a constant as shown in Theorem \ref{th:main}.
\begin{proposition} \label{lem:tightbound}
If $\rho$ is a closed, convex, and differentiable function of $\mathbf{b}$ and $\mathbf{r}_j$ is in the set of all possible values of the gradients of $\rho$ then the maximum error between $\rho(\mathbf{b})$ and $\rho'(\mathbf{b}) \triangleq \max_{\hat{\yvar} \in Y} (\langle \mathbf{b} , \mathbf{r}_{\hat{\yvar}} \rangle - \rho^*(\mathbf{r}_{\hat{\yvar}}) )$ is bounded and positive for $\mathbf{b} \in $ dom $\rho$.
\end{proposition}
\begin{proof}
Since $\rho'(\mathbf{b})$ is the maximum over a family of tangents to a convex function $\rho(\mathbf{b})$ it is guaranteed to be a lower bound to $\rho(\mathbf{b})$. Furthermore, if $\rho'(\mathbf{b})$ is defined for $\mathbf{b} \in $ dom $\rho$ then this error is maximal either at one of the intersection points of the tangents or at the extreme points of the domain of $\mathbf{b}$. In both cases it is finite and positive and can be calculated exactly for given values of $\mathbf{r}_j$ and definition of $\rho(\mathbf{b})$.
\end{proof}

The above proposition bounds the error between a convex function and prediction rewards using its Fenchel conjugate. The Fenchel conjugate is known for several convex functions such as negative entropy (see Property 2), KL-divergence, and $\chi^2$-divergence. Given an arbitrary prediction reward, we can derive exactly how well it approximates a given convex function, such as, negative entropy.

In the rest of this section we perform this analysis for the case where $\rho(\mathbf{b})$ is the negative belief entropy. 
We restrict ourselves to the common reward functions where the agent is rewarded with $r'$ for correctly predicting $\yvar$ and penalized with $r''$ (or not rewarded $r'' = 0$) otherwise, with $r' \geq r''$
\begin{align} \label{eq:rew}
R(\yvar,\hat{\yvar}) &= 
\left\{ \begin{array}{ll}
         r' & \mbox{if $\yvar=\hat{\yvar}, \forall \yvar,\hat{\yvar} \in Y $};\\
        r'' & \mbox{otherwise}.\end{array} \right. 
\end{align}

Using Proposition \ref{lem:tightbound} the difference between $\rho(\mathbf{b})$ and $\rho'(\mathbf{b})$ can be quantified as: 
\begin{equation}
\rho(\mathbf{b}) - \rho'(\mathbf{b}) = -H(\mathbf{b}) - \max_{j \in Y}( \langle \mathbf{b} \mathbf{r}_j \rangle - \rho^{*}(\mathbf{r}_{j}) )
\end{equation}
For the reward defined in \eqref{eq:rew}, $\mathbf{r}_1$ is the vector $(r', r'', r'', \dots, r'')$, $\mathbf{r}_2$ is the vector $(r'', r', r'', \dots, r'')$ and so on. We start by observing that $\rho^*(\mathbf{r}_j)$ is a constant term independent of $j$ and it evaluates to: $\rho^*(\mathbf{r}_1) = \rho^*(\mathbf{r}_2) = \dots = \rho^*(\mathbf{r}_{\numy}) = \log(e^{r'} + (\numy - 1) e^{r''})$.
The term $\max_{j \in Y} \langle \mathbf{b} \mathbf{r}_j \rangle$ can be simplified as max over the following terms  $ \{ (b_1 r' + b_2 r'' + \dots b_{\numy} r'') , (b_1 r'' + b_2 r' + \dots b_{\numy} r'')  , \dots ,  (b_1 r'' + b_2 r'' + \dots b_{\numy} r' )   \}$. Since $b_1 + b_2 + \dots b_{\numy} = 1$ and since $r' > r''$, the maximum over these aforementioned terms is simply equal to:
$\max_{j \in Y} \langle \mathbf{b} ,\mathbf{r}_j   \rangle  = r' \max_{i \in Y} b_i + r'' (1 - \max_{i \in Y} b_i)$.

Using above simplifications $\rho'$ can be written as:
\begin{equation} \label{eq:rho'}
\rho'(\mathbf{b}) = (r' - r'') \max_{i \in Y} b_i + r'' - \log(e^{r'} + (\numy - 1)e^{r''}),
\end{equation}
and the difference between $\rho(\mathbf{b}) - \rho'(\mathbf{b})$ can be characterized as:
\begin{equation} \label{eq:errorChar}
\!\!\!\rho(\mathbf{b}) - \rho'(\mathbf{b}) = \!-H(\mathbf{b}) - (r' - r'') \max_{i \in Y} b_i - r'' \!+ \log(e^{r'} \!\!+ (\numy - 1)e^{r''}). 
\end{equation}

This equation provides the exact error from using the tangents, rather than the negative entropy, and can be queried for a specific $\mathbf{b}$ to provide insights into the level of approximation. We can, however, also  bound this difference for all $\mathbf{b}$, as given in the next theorem.

\begin{theorem}  \label{th:main}
Let $m = r' - r''$ and let\footnote{We can get bounds for $m < 1$ and $m > \numy$, but this introduces more cases and reduces the clarity of the result. We focus the result for the most common $m$. } $1 \leq m \leq \numy$. For every $\mathbf{b} \in [0,1]^{\numy} \mbox{s.t.} \sum_{i \in Y} b_i = 1$,  
\begin{align*}
\rho(\mathbf{b}) &- \rho'(\mathbf{b}) \leq  \max \{\epsilon_1, \epsilon_2 \} + - r'' +\log(e^{r'} +(\numy - 1)e^{r''}) \\
\text{where } \ \  \epsilon_1 &= \log\left(\tfrac{1}{r' - r''} \right) - 1,  \mbox{ and } \ \ 
\epsilon_2 = \log\left(\tfrac{1}{\numy}\right) - \frac{(r' - r'')}{\numy}. 
\end{align*}
\end{theorem} 
\begin{proof}

Starting from \eqref{eq:errorChar}, \\
$\rho(\mathbf{b}) - \rho'(\mathbf{b}) = \!-H(\mathbf{b})-(r' - r'') \max_{i \in Y} b_i - r'' \!+ \log(e^{r'} \!\!+ (\numy - 1)e^{r''})$.
Wlog, let $b_1 = \max_{i \in Y}b_i$, then 
\begin{equation} \label{eq:errorinb1}
\!\!\!\rho(\mathbf{b}) - \rho'(\mathbf{b}) = \!-H(\mathbf{b}) - (r' - r'') b_1 - r'' \!+ \log(e^{r'} \!\!+ (\numy - 1)e^{r''}). 
\end{equation}
For a fixed maximal element $b_1$, the optimal choice to maximize $-H(\mathbf{b})$ is to concentrate the remaining probability mass on as few elements as possible subject to constraints that $b_i \leq b_1$ for $i \neq 1$ and $i \in Y$. This means setting $b_2 = 1-b_1$ if $b_1 > 0.5$. Of course, $b_1$ might be less than 0.5. In general, for some $k \ge 1$, we set $b_{1:k} = b_1$ and then $b_{k+1} = 1- k b_1$ for the remaining probability. 
The resulting $-H(\mathbf{b}) =  kb_1 \log(b_1) + (1 - kb_1)\log(1 - kb_1) $  upper bounds the negative entropy for any distribution with max element $b_1$. 

For $m \doteq r' - r'' \geq 0$, define
\begin{equation*}
g(b_1) \doteq kb_1 \log(b_1) + (1 - kb_1)\log(1 - kb_1) - m b_1
\end{equation*}
 where $\numy \geq k \geq 1$ and $b_1 \in [\frac{1}{\numy}, \frac{1}{k}]$. Finding $b_1$ that is maximal for $g$ will will be the same $b_1$ that is maximal for the rhs of $\eqref{eq:errorinb1}$ and so give an upper bound on $\rho(\mathbf{b}) - \rho(\mathbf{b}')$. 
Therefore, we only need to find an upper bound on $g(b_1)$ to prove the theorem. 
First, we know that $g(b_1)$ is a convex function for $b_1$ where $\frac{1}{\numy} \leq b_1 \leq \frac{1}{k}$ because 
\begin{align*}
g'(b_1) &= k + k \log(b_1) - k \log(1 - k b_1) - k - m   \\
&= k\log(b_1) -k\log(1 - kb_1) - m,
\end{align*}
and 
\begin{align*}
g''(b_1) &= \frac{k}{b_1} - \frac{k}{1-kb_1}(-k)   \\
& = \frac{k}{b_1} + \frac{k^2}{1 - kb_1} > 0
\end{align*}
Therefore $g(b_1)$ is maximal at the endpoints $b_1 = \frac{1}{\numy}$ or at $b_1 = \frac{1}{k}$, where $\numy \geq k \geq 1$.

If $b_1 = \frac{1}{k}$ ($b_1 \to \frac{1}{k}$ to be more precise), then
\begin{align*}
g\left(b_1=\tfrac{1}{k}\right) = \log\left(\tfrac{1}{k}\right) + 0 - \frac{m}{k}   
\end{align*}
We can again reason about this function, and find the $k$ that makes this maximal and so provides an upper bound on $g$. Let $f(k) \doteq \log\left(\tfrac{1}{k}\right) - \frac{m}{k}$. $f'(k) = -\tfrac{1}{k} + \tfrac{m}{k^2} = 0$ gives $k = m$. Further, for $1 \le m \le \numy$, we know this function is concave for the region $0 \le k \le 2m$ because $f''(k) = \tfrac{1}{k^2} - \tfrac{2m}{k^3} < 0$ if $k \le 2m$. Since this stationary point $k = m$ is in this concave region, we know it is a local maxima. Further, for $k > 2m$, the function becomes convex, but only decreases because there is no stationary points other than $k = m$. Therefore, for this case, the maximal $g$ is 
\begin{equation*}
\epsilon_1 = \log\left(\tfrac{1}{m}\right) -1
.
\end{equation*}

If $b_1 = \frac{1}{\numy}$, then 
\begin{equation*}
\epsilon_2 = g\left(b_1=\tfrac{1}{\numy}\right) = \log\left(\tfrac{1}{\numy}\right) - \frac{m}{\numy}.
\end{equation*} 
Putting it all together, since we found $\max(\epsilon_1,\epsilon_2)$ as an upper bound on $g(b_1)$ for all $b_1$, we get that 
\begin{align*}
\rho(\mathbf{b}) - \rho'(\mathbf{b}) &= g(b_1) - r'' + \log(e^{r'} + (\numy - 1)e^{r''})\\
&\le \max(\epsilon_1, \epsilon_2) - r'' + \log(e^{r'} + (\numy - 1)e^{r''}).
\end{align*}
\end{proof}

\begin{corollary}[0-1 Prediction Rewards]  
If $r' = 1$ and $r'' = 0$, then for every $\mathbf{b} \in [0,1]^\numy$ s.t. $\sum_{i \in Y} b_i = 1$, 
\begin{align*}
\rho(\mathbf{b}) - \rho'(\mathbf{b}) & \leq  -1 + \log(e + \numy - 1) .
\end{align*}
\begin{proof}
Direct application of Theorem \ref{th:main}. Substituting $m = r' - r'' = 1 - 0 = 1$, we get $\epsilon_1 = -1$ and $\epsilon_2 = \log(\frac{1}{\numy}) - \frac{1}{\numy}$. Since 
Since $-1 \geq \log(\frac{1}{\numy}) - \frac{1}{\numy}$ for $\numy \geq 1$, and substituting $r' = 1$ and $r'' = 0$, we get $\rho(\mathbf{b}) - \rho'(\mathbf{b}) \leq  -1 + \log(e + \numy - 1)$.
\end{proof}
\end{corollary} 

\subsection{Consequences of the Theory}

\textbf{Computing the optimal action} The previous results showed that $\rho'(\mathbf{b})$ $= \max_{j \in Y} \langle \mathbf{b}, \mathbf{r}_j \rangle - \rho^*(\mathbf{r}_j)$ is an approximation to $\rho(\mathbf{b})$ if $\rho$ is convex. Fortunately, to compute the action $a^{*,t}$ that  maximizes the information gain of the agent we do not need to compute $\rho^*(\mathbf{r}_j)$ as it is independent of the actions and is a constant for a fixed $j = \arg\max_{j \in \stateset} \langle  \mathbf{b}, \mathbf{r}_j \rangle - \rho^{*}(\mathbf{r}_j)$ equal to $\log(e^{r'} + (\numy - 1)e^{r''})$ (for reward defined in \eqref{eq:rew}). The agent can approximate $a^{*,t} = \arg\max_{a \in \actionset}\mathbb{E}[H(p^{t}) - H(p^{t+1})]$ (here $p^{t+1}$ depends on $a$) by picking actions that maximize $\mathbb{E}_{\Pr(z^{t+1}|p^{t}, a)}[\max_{\hat{\yvar} \in Y} \sum_{\yvar} p^{t+1}(\yvar) R(\yvar,\hat{\yvar})]$ or an sample estimate of it. This sample estimate can be computed without maintaining an explicit distribution $p^t$ but instead by training an agent to make correct predictions based on history of action and observations. In the next section we do exactly that.

\begin{figure}
\begin{center}
\includegraphics[scale=0.38]{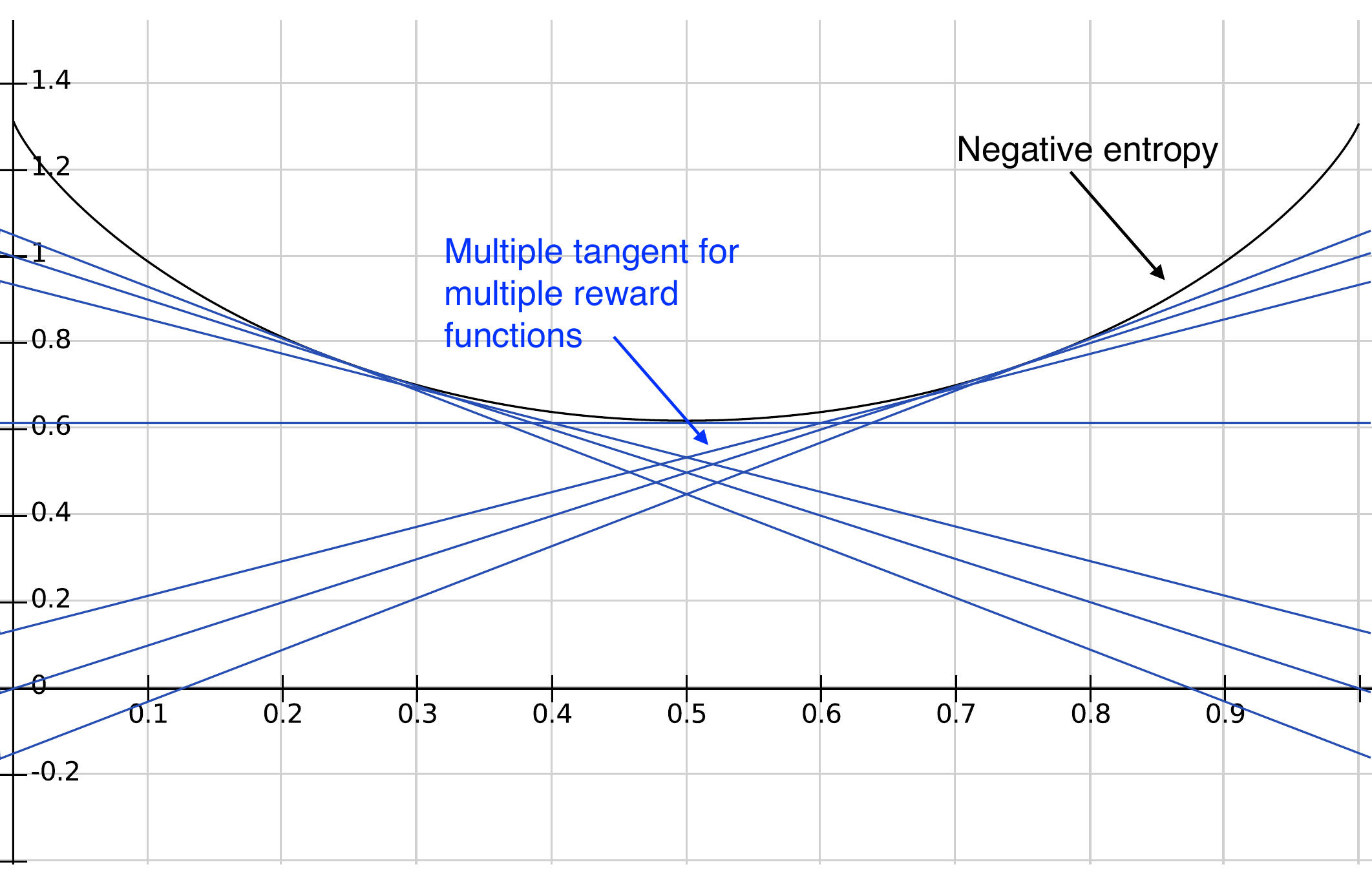}
\end{center}
\caption{Approximation induced by prediction rewards to a translated negative entropy curve.} \label{fig:tangents}
\end{figure}
\textbf{Reducing the error to zero:}  The error between prediction reward and information gain can be further reduced by giving the agent the choice of selecting from one of many prediction variables, each of which defines a separate prediction reward as shown in Figure \ref{fig:tangents}. To do so we define multiple prediction reward $R^{l}(\yvar,\hat{\yvar}^l)$, each of which takes as input a separate prediction variable $\hat{\yvar}$. Furthermore, define $\rho'(\mathbf{b}) = \max_{ \{l, j\} \in \{M \times Y\}} ( \langle \mathbf{b} , \mathbf{r}^{l}_{j} \rangle - \rho^{*}(\mathbf{r}^{l}_{j})    )$, where $M$ is the set of all values $l$ can take (4 in this case). 
Each of these reward functions projects a tangent (or tangent hyperplane) to the original $\rho$, in this case the entropy, with $\hat{\yvar}^{4}$ (corresponding to the blue tangent line parallel to x-axis) being unique in that it rewards the agent equally for correct or incorrect predictions. In this way, $\hat{\yvar}^4$ offers the agent an the option to abstain, which is optimal when it is most uncertain (bottommost point of the negative entropy curve). As more and more tangents are defined using new prediction variables, the upper surface of the tangents can approximate the original $\rho$ more and more closely.

\subsection{Connection to Existing Literature}
An important consequence of this section is that it ties the problem of maximizing information gain \citep{Krause05uai, Nowozin12, Yang16, SatsangiJournal17} to many recent deep RL approaches, that are based on making a correct predictions at the end of an episode \cite{Mnih14, Nara16, Jaderberg16, Oh16, Mousavi19}. For example, both visual attention approaches \citep{Mnih14, Haque16, Mousavi19} and question answering systems \cite{Nara16} train deep RL agents on a 0-1 prediction reward for classifying an image and answering a query correctly respectively. Visual attention, question answering systems, intrinsic motivation, active perception, sensor placement, and active sensing are separate sub-fields of artificial intelligence, that do not necessarily refer to each other very often, however, our results show that they are in fact solving the same problem (or a close approximation of it). 

Our theoretical results are related to $\rho$POMDPs \cite{Araya10} and POMDP-IR \cite{Spaan15} and their equivalence as established in \cite{SatsangiJournal17}. This works shows that given a $\rho$POMDP---which has a reward function defined by a set of vectors that approximate a convex curve---it is possible to design an equivalent POMDP-IR with a prediction reward. However, they do not give any direction as to how to compute the vectors that closely approximate the convex curve. We circumvent the procedure of computing these vectors by using the theory of Fenchel conjugates that gives us direct and analytical expressions for computing the tangent hyperplanes to a convex curve. Consequently, we are able to derive the exact error bound caused by a prediction reward, for example, a 0-1 prediction reward. 

\section{Deep Anticipatory Networks}
The insights in the previous section motivate that we no longer need an explicit belief to evaluate the information gain of an action, and can instead employ existing deep RL algorithms such as deep $Q$-learning to learn a policy that maximizes prediction rewards. In this section we introduce \emph{deep anticipatory networks} (DANs), an algorithm that enables an agent to take actions that help it predict the current and future values of $\yvar$ accurately. DAN consists of two different networks: a $\mathsf{Q}$ network and a model $\mathsf{M}$ network. The $\mathsf{Q}$ network takes as input the  action-observation history $h^{t} = \langle a^0, z^1, a^1, \dots, a^{t-1}, z^{t} \rangle$ of the agent and outputs the $Q$-values of all available actions. The agent takes an action $a^t$ ($t$ denoting the current time step) that maximizes the $Q$-values and receives an observation $z^{t+1}$ that is correlated with the unknown variable $\yvar$ at time step $t+1$. This new action-observation pair is added to the history and fed into the $\mathsf{M}$ network.

The $\mathsf{M}$ network takes as input the agent's action-observation history and predicts the value of the unknown variable.
The $\mathsf{M}$ network is trained in a supervised fashion using the agent's dataset of action-observation histories labelled with the corresponding true $\mathbf{Y}$.  If the $\mathsf{M}$ network predicts the state of the world correctly, then the $\mathsf{Q}$ network is rewarded +1 and otherwise 0. In other words, the $\mathsf{Q}$ network is rewarded for learning a $Q$-function that takes actions that help the model to predict the  state from partial observations. 
Figure \ref{fig:abstractDAN} illustrates an abstract DAN. 

To train DAN, both the $\mathsf{Q}$ and the $\mathsf{M}$ networks are trained simultaneously on small mini-batches of data. Since one of the components in DAN is DQN, we additionally borrow the techniques used to train DQNs to train DAN. Specifically, each history-action pair that the agent encounters is stored in an experience buffer to be sampled later to train both the $\mathsf{Q}$ and the $\mathsf{M}$ networks. We maintain two separate target networks for $\mathsf{Q}$ and $\mathsf{M}$ networks to get stable target values when updating the $\mathsf{Q}$ network.

In each iteration, for each episode, the agent follows the policy that is greedy with respect to the $Q$-values of the $\mathsf{Q}$ network. The accumulated experience is added to the experience buffer in the form of the tuple $\langle h^{t}, a^{t}, r^{t+1}, h^{t+1}, \yvar^{t+1} \rangle$ that is later used to train the $\mathsf{Q}$ network. The observations ${z}^{t+1}$ and the true $\yvar^{t+1}$ are obtained from the dataset while the reward $r^{t+1}$ is obtained from the target $\mathsf{M}$ network. At each time step, the agent samples random experience tuples from the experience buffer and updates $\theta_{\mathsf{Q}}$ using a Q-learning update, with a target network.
Once $\theta_{\mathsf{Q}}$ is updated, $\theta_{\mathsf{M}}$ is updated by gradient descent with a cross-entropy loss: $\theta_{\mathsf{M}} = \theta_{\mathsf{M}} + \alpha 	\nabla_{\theta} L_{\mathsf{M}}(\theta_{\mathsf{M}})$, where $L_{\mathsf{M}}(\theta_{\mathsf{M}}) = \mathtt{cross\hbox{-}entropy}( \mathsf{M}(h^{t}|\theta_{\mathsf{M}}), \yvar)$.

\begin{figure}
\includegraphics[scale=0.47]{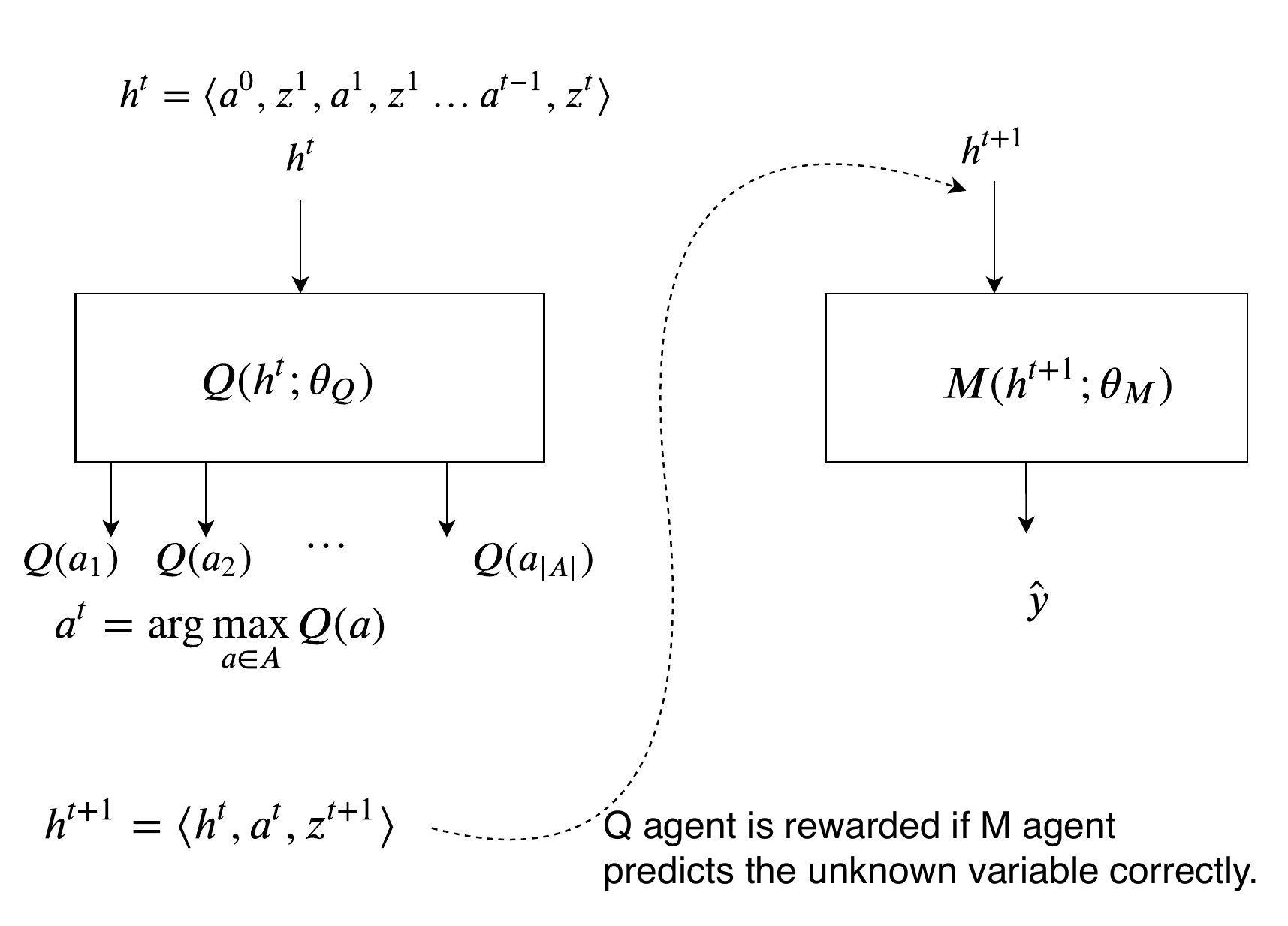}
\caption{An abstract model of DAN that consists of a $\mathsf{Q}$ network and an $\mathsf{M}$ network. The $\mathsf{Q}$ network controls the input to the $\mathsf{M}$ network and the $\mathsf{M}$ network controls the reward the $\mathsf{Q}$ network gets. } \label{fig:abstractDAN}
\end{figure}

The idea of learning sensory actions $(\mathsf{Q})$ and a predictive model ($\mathsf{M}$) simultaneously have appeared in earlier literature, with \citep{Mnih14} the closest of all architectures. Similar architecture are presented in \cite{Haque16, Bachman17,Mousavi19}. The specific architectures in \citep{Mnih14}, \citep{Haque16}, \citep{Bachman17} and \citep{Mousavi19} differ, but they share a common idea: to train the neural network architecture with policy gradient methods on a single unified objective, for example, using REINFORCE \citep{Williams92} or proximal policy optimization \citep{Schulman17}. We chose to use DQN, particularly because it facilitates the use of factorization of the state-space and because we use knowledge of the exact action-values for the sensor selection system. 

Otherwise, this choice is not critically different: either policy gradient methods or Q-learning methods can be used to solve this problem. A more interesting distinction is in the fact that the DAN architecture makes the it clear how general RL problem definitions can be used. It is common to model the problem of classification as a terminal-reward problem where the agent is rewarded only at the end of the episode (after a fixed number of steps). This is applicable when $\yvar$ is not changing with time. We explicitly formulate this problem as a continual problem where the agent is rewarded at each time step if it correctly prediction the unknown variable $\yvar$. Such a formulation is critical when $\yvar$ changes with time, for example, in the sensor selection problem. But even in cases when $\yvar$ does not change with time, our experiments suggests that providing feedback on every step leads to faster learning. This has important implications for training visual attention and question answering systems. 

\section{Experiments} 
In this section we present two different applications of DAN: sensor selection for tracking people in a shopping mall and discrete visual attention for classifying MNIST digits. Code for our experiments is available online.\footnote{\url{https://github.com/sungsulim/DeepAnticipatoryNetworks}}

We apply DAN to build a sensor selection system that we demonstrate can scale to arbitrarily large  spaces. We use DAN to learn a sensor selection policy to track people in a shopping mall.  The problem was extracted from a real-world dataset collected in a shopping mall \cite{Bouma13}. The dataset was gathered over 4 hours using 13 CCTV cameras. Each person's position is represented by $x$-$y$ coordinates, where both $x$ and $y$ take values in the set $\{1, 2, \dots 50\}$ resulting in a total of $50 \times 50$ cells. At each time step, the agent selects one camera out of 10 to get an observation about the location of the person in the image. Each camera covers a subset of $50 \times 50$ cells and provides a noisy observation regarding the position of the person. If the person is not present in the image then a null observation is received. This observation along with the selected camera is passed to the $\mathsf{M}$ network that predicts which of the $50 \times 50 (=2500)$ cells the person occupies.
\subsection{Sensor Selection}
The number of states of the world increases rapidly with the number of people in the scene. To address this, we assume that the movement of a person in the $x$-direction is independent of his/her movement in the $y$-direction and vice-versa. We train two separate DAN architectures, DAN-x and DAN-y for separately predicting the $x$ and $y$ coordinates of the position of a person. Furthermore, we assume that the movement of people present in the scene is independent of each other. These approximations let us build a sensor selection system that can scale to larger spaces and numbers of people. 

For sensor selection, both the $\mathsf{Q}$ and $\mathsf{M}$ networks share an identical architecture: three fully connected layers of output size 60, 30, and 128, followed by a recurrent layer of output size 128, and a final fully connected output layer of size 10 (the number of cameras) and 51 (the number of possible cells + null observation). Strictly speaking, here we are using deep recurrent Q network (DRQN)\citep{Hausknecht15} in the DAN architecture instead of DQN. We use ReLU activation for all fully connected layers except the last, and use L2 weight regularization (scale=0.01). We use the discount factor $\gamma = 0.99$ and perform a double DQN \citep{Hasselt16} update to train $\mathsf{Q}$ network with the Adam optimizer \cite{Kingma14}. 
We also train following baselines for comparison. 
\textbf{Coverage baseline} --- train only the $\mathsf{Q}$ network using the popular state-based reward (i.e., reward the agent for selecting the camera corresponding to the person's current location and getting a positive observation) without the $\mathsf{M}$ net. It uses its observations as final predictions, and during evaluation the agent only has to obtain a positive observation to be considered to have made a correct prediction.
\textbf{Random Policy baseline} --- only train the $\mathsf{M}$ network with a random policy for camera selection. 
\textbf{DAN + Coverage baseline} --- use a combination of DAN reward and coverage reward, in which case the agent is rewarded +1 for correctly predicting the state, +0.2 for not being correct but getting a positive observation, and 0 otherwise (but we still use the $\mathsf{M}$ network to predict the $x$ and $y$ coordinates). \textbf{DAN-shared} is when the Q and M networks share  representations, that is the top layers share the same parameters for both the Q and M network, but the last layer is separated.

\begin{figure}
\begin{center}
\includegraphics[scale=0.28]{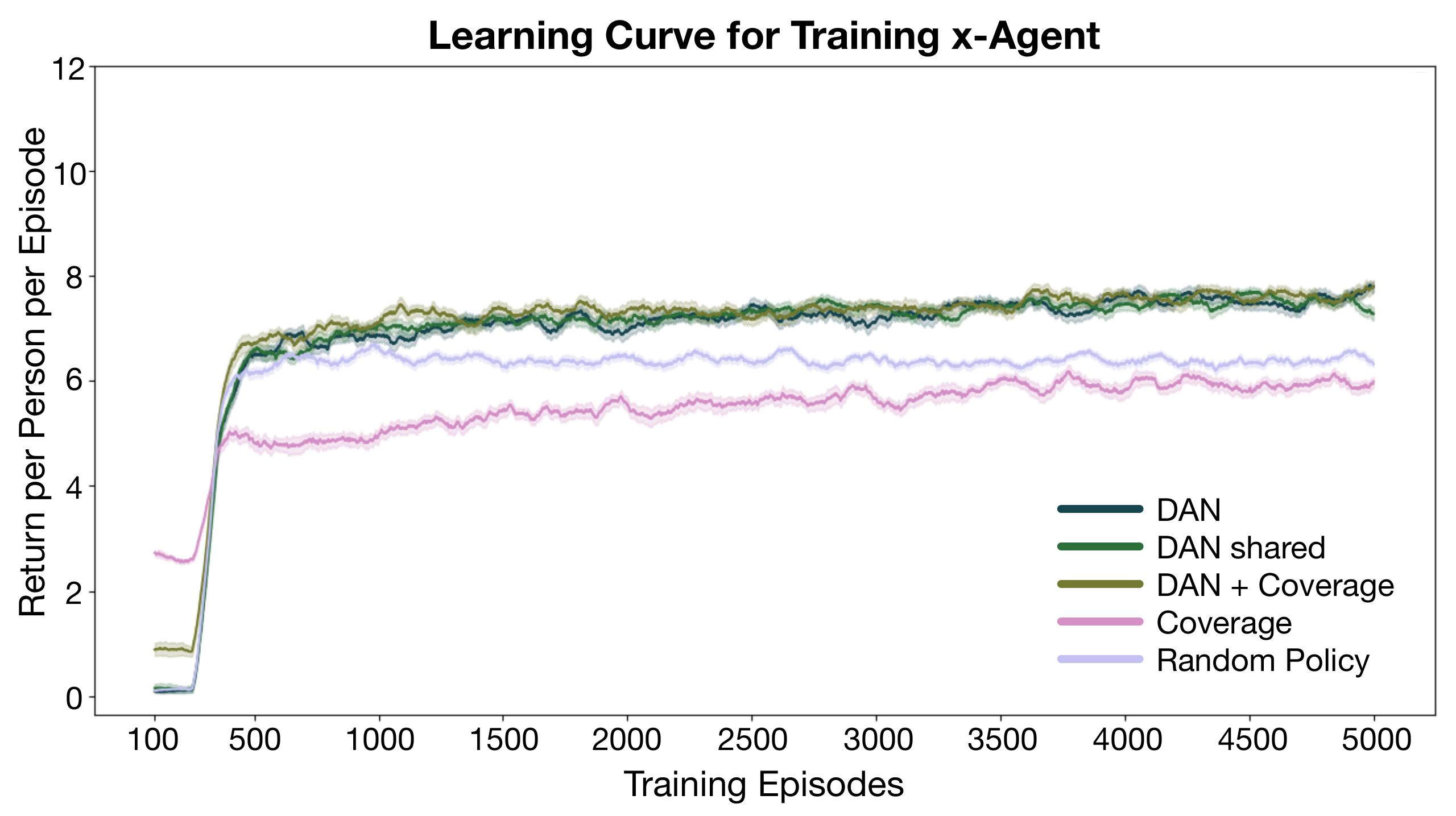}   
\includegraphics[scale=0.28]{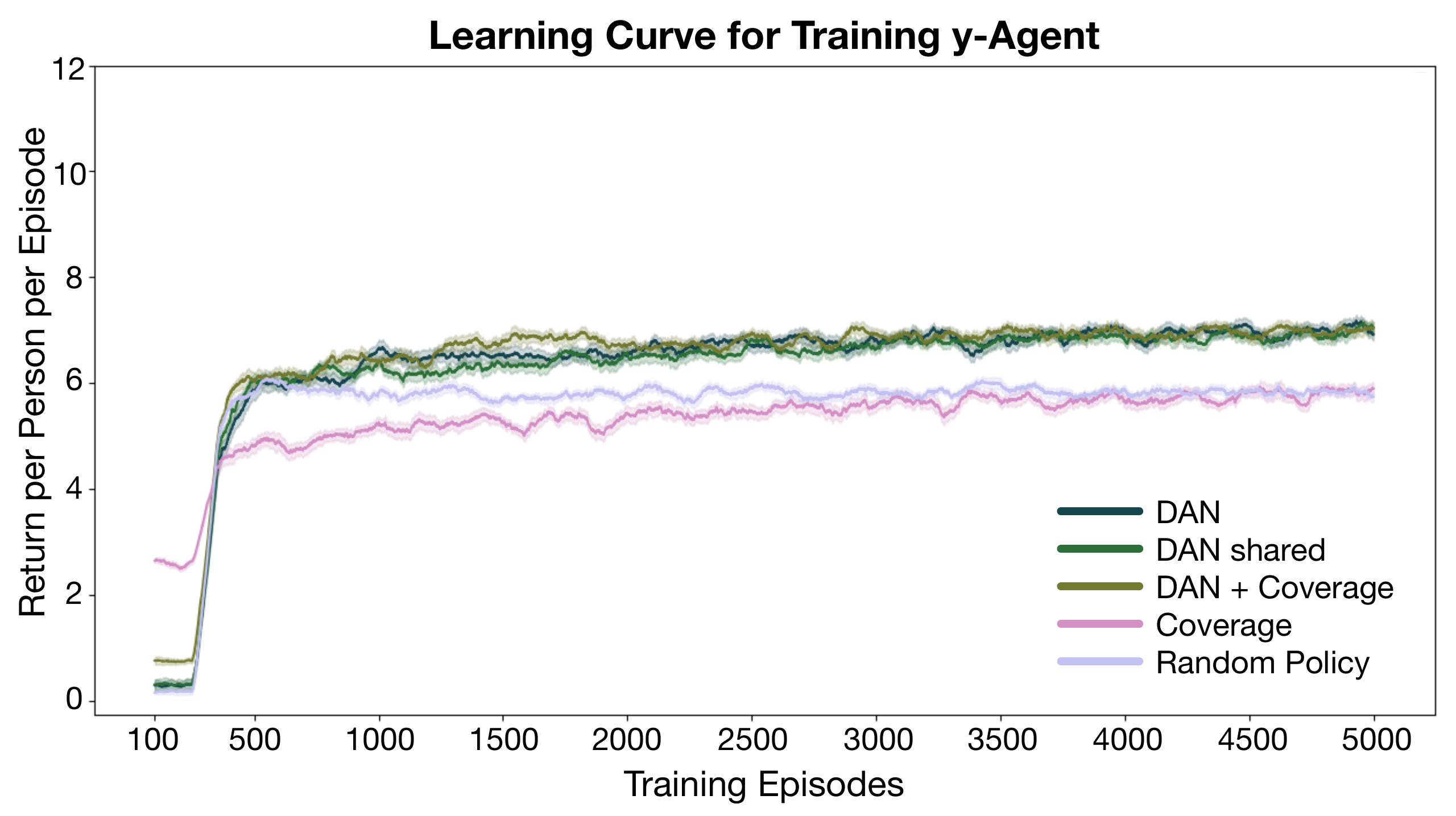}   
\includegraphics[scale=0.28]{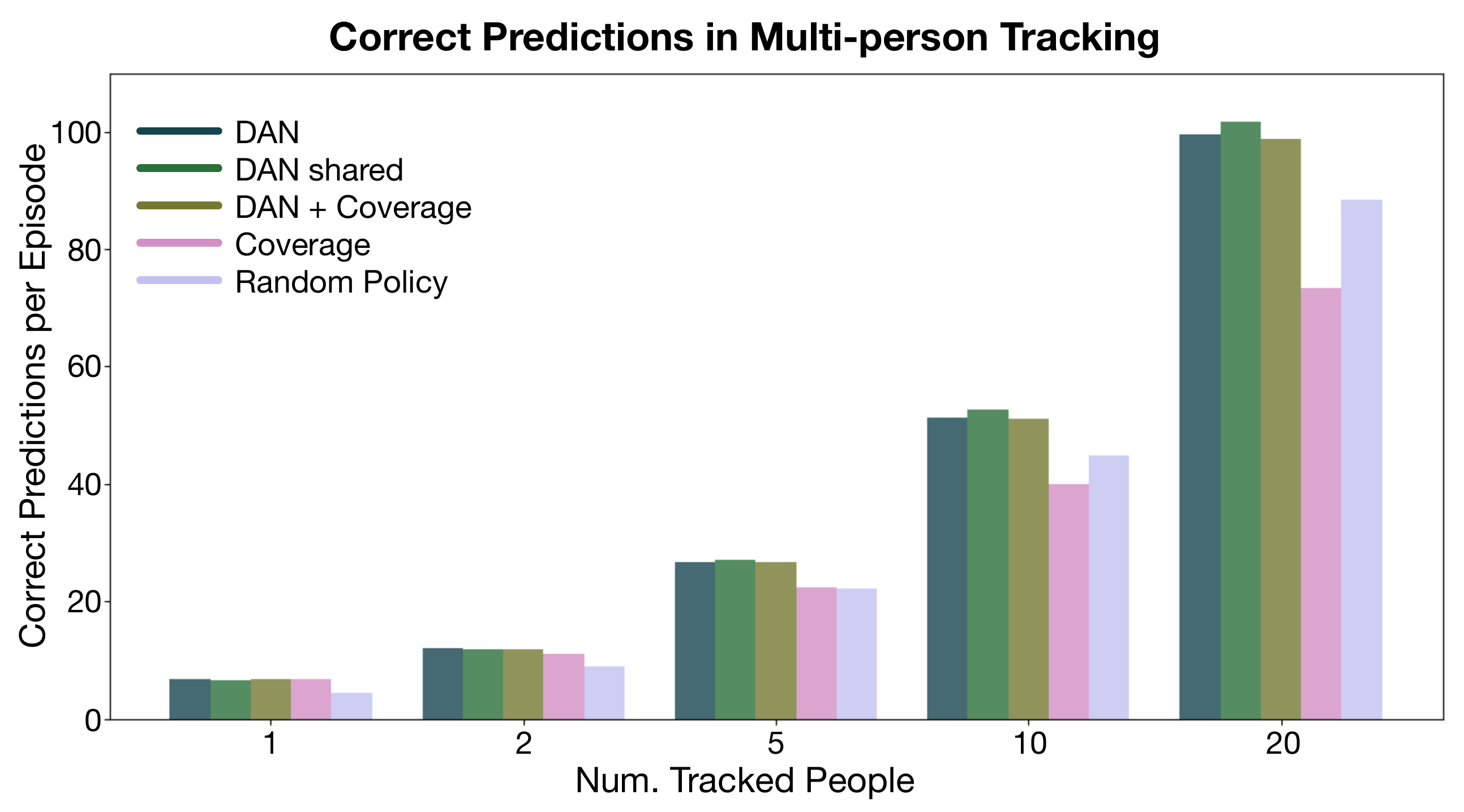}
\caption{Training curves and multi-person tracking results for sensor selection for DAN agent.} \label{fig:ssRes}
\end{center}
\end{figure}

We also compare to a model-based particle filter approach and to a DAN model that is trained on a terminal reward (only provided at the end of the episode during training) instead of a continuous reward that is provided at each time step of the training. However, these two baselines performed particularly poorly. The particle filter based approach that had access to the learned transition dynamics (under Gaussian assumption) and the true observations noise results in a performance of 1.9 (less than 1/3 of DAN's) total reward per trajectory for 400 particles and saturates at 3.5 (less than 1/2 of DAN's performance) for 1500 particles and after tuning many parameters of the particle filter. Rewarding an agent only at the termination of the episode does not work either as for tracking the agent needs continuous feedback. We did not experiment further with these baselines.

For training DAN and baseline methods, we swept over the exploration probability $\epsilon: \{0.1, 0.3, 0.5\}$ and $\mathsf{Q/M}$ network learning rate: $\{0.01, 0.001, 0.0001\}$. For all methods we found $\epsilon=0.1$ and $\mathsf{Q/M}$ network step-size $=0.001$ to work the best. We first train $x, y$-agents for tracking a single person, and the training curves are shown in Figure \ref{fig:ssRes} (a) and (b). We perform 25 runs for each agent. 
We use track length of 12, sampled from the training track dataset, and train it for 60,000 steps (or 5000 episodes). We also collect experience without training for 3,000 steps (250 episodes). For updating the networks, we use a mini-batch of size 4 to sample episodes from the replay buffer, with trace length of 8 (not updating on the first 4 steps of the episode).

We test the trained DAN agents in single-person and multi-person tracking. For single person tracking, at each time step the agent queries the $Q$-values from both the DAN-$x$ agent and DAN-$y$ agent and selects the camera (action) that maximizes the average $Q$-value among all the available actions. For multi-person tracking we transfer the policy learned for single-person tracking to track multiple people. The same $\mathsf{Q}$ network is used to compute the $Q$-values of selecting each camera for each person independently. Finally, the agent selects the camera that maximizes the average $Q$-value from all the people present in the scene, and the $\mathsf{M}$ network predicts the location of all the people based on the observation. During evaluation the agent is rewarded +1 only if both $x$ and $y$ coordinates are predicted correctly. Figure \ref{fig:ssRes} (c) shows the result of multi-person tracking of 500 test tracks. In all cases, variants of DAN outperform the random and coverage baselines. Surprisingly, sharing representation is comparable to DAN with separate representations for $\mathsf{Q}$ and $\mathsf{M}$ networks, which is good as sharing representations reduces the number of parameters.

 \subsection{Discrete attention}
In this set of experiments, we apply DAN to learn discrete models of attention in which the agent can observe the unknown variable only via a discrete set of available glimpses. As compared to sensor selection here the hidden variable is not changing and selecting one of the available glimpse does not necessarily provides the agent enough information for predicting the digit in the image. So ideally the agent must learn representation that help it predict the digits from as little glimpses as possible. At the start of the episode the agent receives a blank image and as it makes its selections, glimpses of the images are revealed. This task is discussed in earlier papers \citep{Mnih14} with different glimpse styles depending on the motivation of the paper. However, many earlier approaches based on deep reinforcement learning model this task with a terminal reward the agent receives the feedback (reward and true label) about its policy only at the end of the episode. Our formulation models this as a continuous feedback task, where the agent makes a prediction at each time step and is rewarded at every time step for making correct predictions. Since during the training the true label is available to the agent, there is no point of making this label available to the agent only at the end of the episode. 

For this experiment, the $\mathsf{Q}$ and $\mathsf{M}$ networks are identical convolutional neural networks (CNN) with two convolutional layers. This is followed by a max pooling layer and two fully connected layers with a dropout \cite{Srivastava14} probability of 0.5. ReLUs are used as activation units for all layers. The length of the episode is kept to 12 and the networks are updated every 4 steps. A learning rate of 0.0005 (after performing a parameter sweep over \{0.05, 0.005, 0.0005\}) is used with the Adam optimizer \cite{Kingma14}. An exploration probability of 0.05 is used throughout training but an exploration probability of 1 is used during the first 1500 episodes.

\begin{figure}
 \centering
  \includegraphics[scale=0.35]{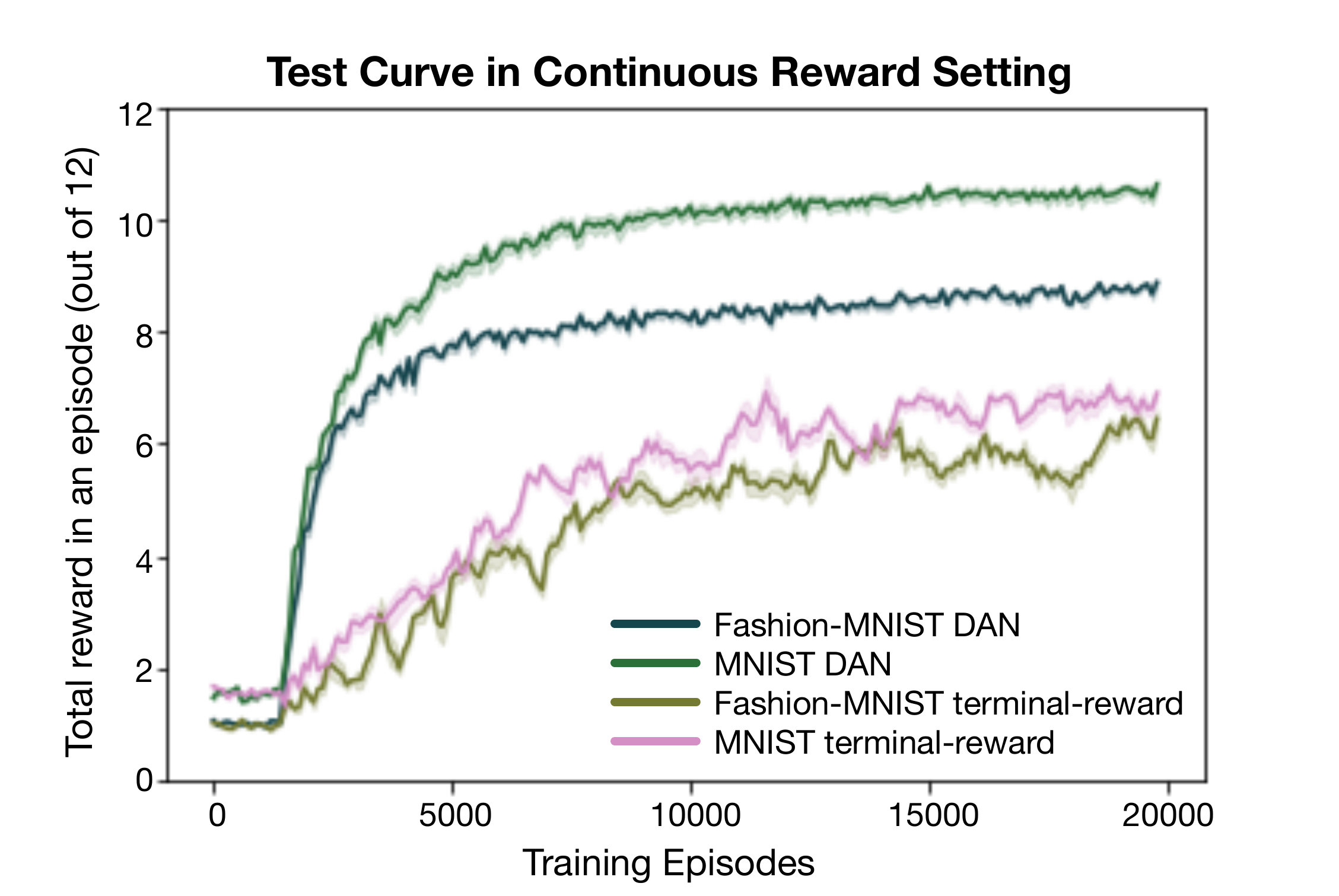}  \\
  \includegraphics[scale=0.35]{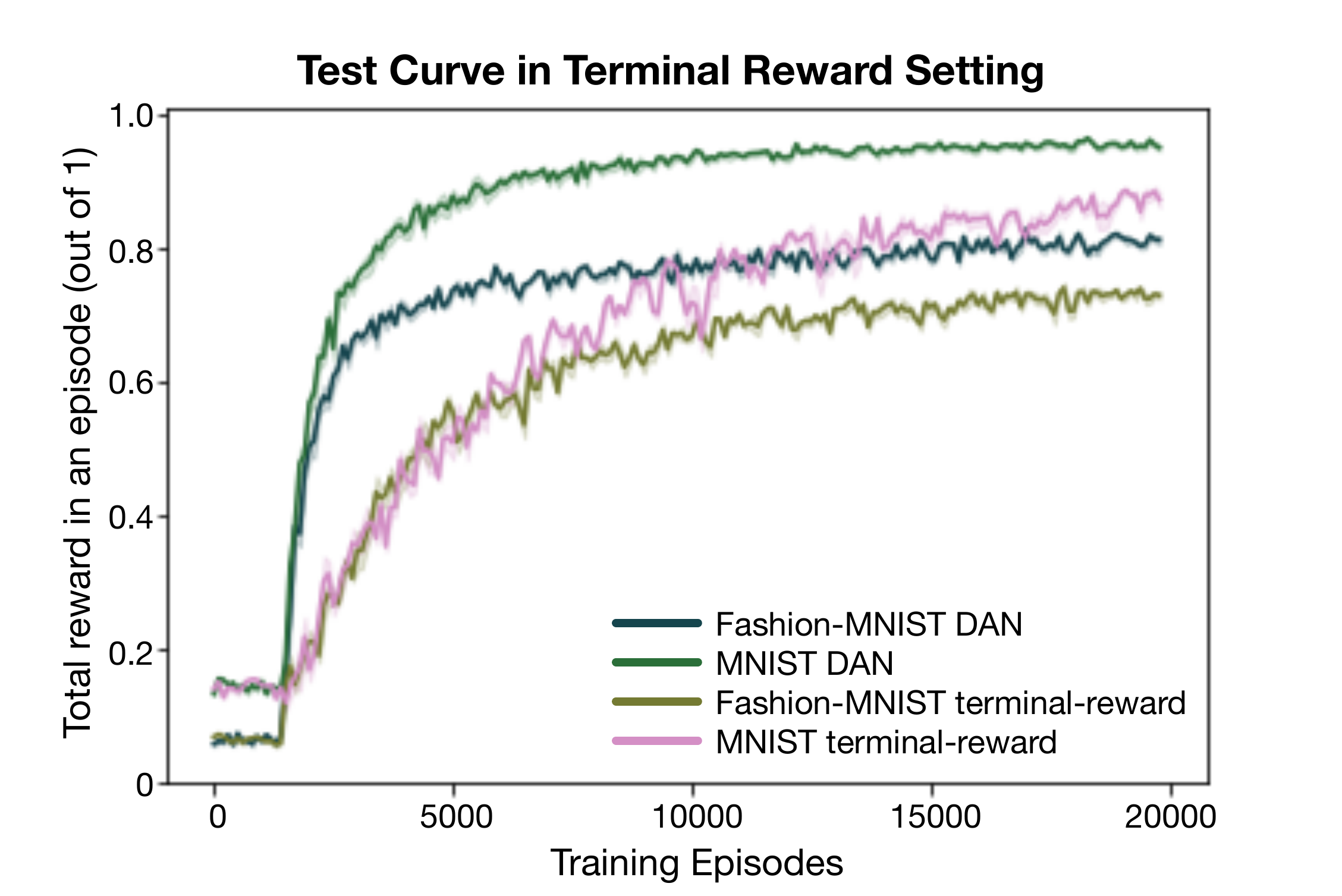} \\
  \includegraphics[scale=0.14]{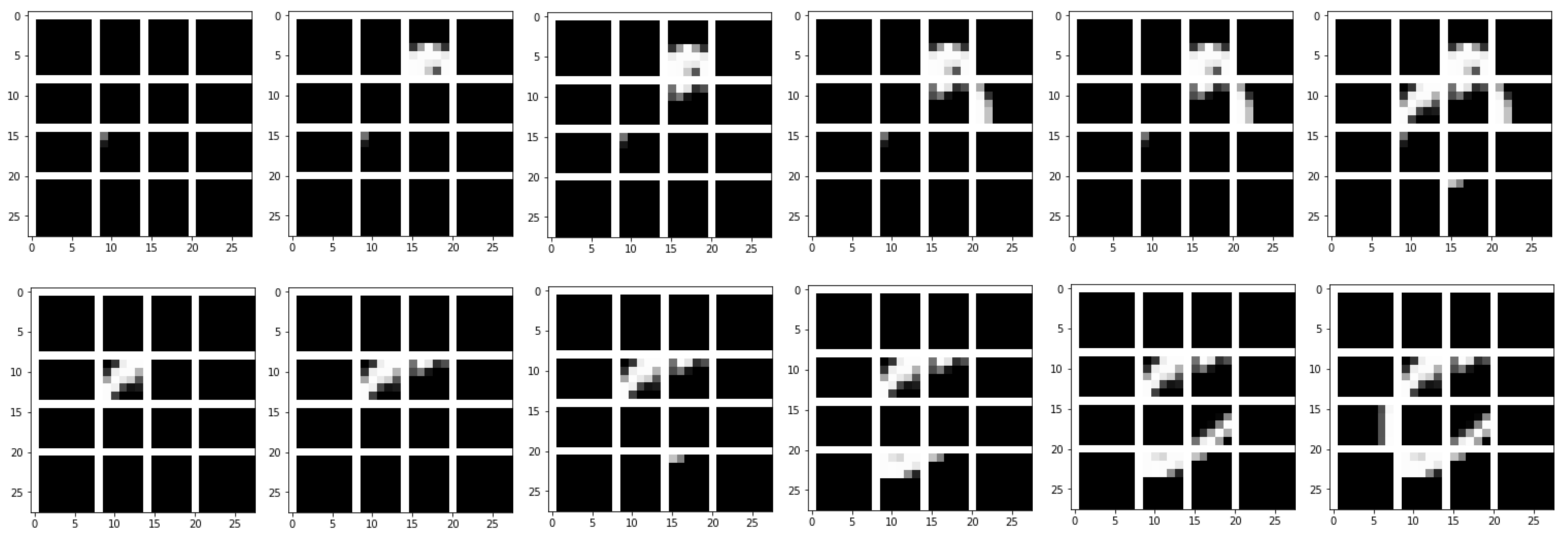}
  \caption{(top, middle) Performance results for discrete attention in continuous/terminal reward setting averaged over 10 runs, (bottom) Sequence of MNIST glimpses selected by the DAN agent for two separate examples.} 
  \label{fig:mnistPerf}
  \end{figure}

We compare and evaluate DAN trained with continuous reward and DAN trained with terminal reward in two different setting (a) continuous reward and in (b) terminal reward settings. For the evaluation in the continuous reward setting the agent is rewarded at each time step for an episode of length 12 (so the agent can earn a maximum reward of 12) where as in the terminal reward setting the agent is evaluated on a terminal reward that the agent receives at the end of the episode. Figure \ref{fig:mnistPerf} shows the average test reward on 500 test images (sampled from a set of 10000 test images at every evaluation) as a function of the training episode for both MNIST and fashion MNIST. The top figure shows the results for when the agent is rewarded at each time step and the middle figure shows results when evaluating on a terminal reward. In both settings the agent trained on continuous reward is significantly faster than the terminal reward setting simply because (a) it is simultaneously trained to select glimpses that can most quickly identify the classes as well as to identify classes from as few glimpses as possible; (b) it better uses the same set of experience to make more updates to its parameters  because of the continuous feedback. DAN with terminal rewards performs particularly poorly in the continuous reward setting, as the $\mathsf{M}$ network in the terminal reward DAN is not trained to predict the class from smaller number of glimpses. Furthermore, the results also show that, at least for MNIST, it is possible to identify the digits from only one or two glimpses, as the DAN agent gets an average reward of more than 10 out of 12 on test images, whereas for the fashion MNIST, correctly predicting the right class requires a couple of more glimpses.

\section{Related Work}
Prediction rewards are popular in reinforcement learning, for example, visual attention models \citep{Mnih14, Haque16}, question answering systems \citep{Nara16, Buck18}, learning active learning strategies \citep{Bachman17}, intrinsic motivation \citep{Pathak17}. On the other hand, literature such as active perception \citep{SatsangiJournal17}, sensor placement \citep{Krause05uai}, and active sensing \citep{Kreucher05}, formulate the problem of either sensor management/selection/fusion with information gain as the objective function. Our paper ties these fields together by exactly establishing the relationship between prediction rewards and information gain. 

Model-based methods as proposed in various active perception \citep{Allen85, Kelly70, Ye95, Burgard97, Wilkes92, Bruce09, Bajcsy16} and sensor selection \citep{Kreucher05,Williams07,Spaan09,Joshi09,Hero11,Monari10,Tessens14} literature require a model of the world for their application. The model-free nature of DAN lets us to deploy deep RL machinery for sensor selection in a principled manner. Recently, attempts to perform \emph{online active perception} \citep{Ghasemi19, Mousavi19} either focus on fast subset selection or on neural network architecture improvement, e.g., for MNIST, but offer no insight on connecting prediction rewards to information gain.

Neural models of visual attention, such as that of \cite{Mnih14} and \cite{Haque16}, consider a classification task where the unknown variable is not changing at every time step. Consequently they model the loss function as one conditioned on a terminal reward that the agent receives if it correctly classifies the image after certain time steps. By contrast, sensor selection is a continual learning setting where the position of the person is continuously changing and the agent must predict it at each time step using noisy observations. Moreover, the agent in the classification task is free to adjust the size and shape of the glimpse. By contrast, in sensor selection the agent can only attend to the scene with a fixed (already deployed) set of glimpses that cannot be resized. 

Approaches that use intrinsic motivation \citep{Pathak17} and auxiliary tasks \citep{Jaderberg16} use the prediction reward as a means to train an agent to solve a specific task. The performance of the policy is evaluated on an \emph{extrinsic} state-based reward; the goal is not prediction accuracy. 
By contrast, our aim is to maximize the prediction reward and not use it achieve any other target. 

DANs are related to learning in POMDPs/MDPs \citep{James09,Katt17} but are designed to learn hidden representations of the world as opposed to the transition or observation function after assuming/designing the representation of the world. 
Generative adversarial networks (GANs) \citep{Goodfellow14} and DIAYN \citep{Eysenbach18} train two different networks on each other's feedback. However, GANs assume an adversarial relationship between the two networks leading to a min-max formulation of the final objective, while DANs lead to max-max formulation of the final objective.  DIAYN \cite{Eysenbach18} consists of two networks, one of which tries to help the other discriminate between objects in order to learn various skills, whereas our aim is to predict the unknown variable and maximize the prediction reward in itself. 

Neural estimators based on variational lower bound to KL divergence \citep{Mohamed15, Belghazi18} do not acknowledge the connection between prediction rewards and negative entropy as we do. These approaches also do not categorize the error between the variational lower bound and information gain as we do, which can be further exploited to vanish this error. Thanks to the theory of convex duality, our insights are extendible to any convex functions of the belief and not just KL-divergence. Furthermore, these approaches propose an estimator but do not demonstrate the use of these estimator in a partially observable setting for sensor selection as we do. 

Our results are also related to $\rho$POMDP \cite{Araya10} and POMDP-IR \cite{Spaan15} and their equivalence as established in \cite{SatsangiJournal17}. Apart from the distinction made earlier in Section 3, this paper present a deep reinforcement learning algorithm as compared to a model-based planning method they propose. Approaches \citep{Le08,Kostrikov16} that model active perception tasks with surrogate state-based rewards are fundamentally different from our formulation because of the definition of the reward.

\section{Conclusions \& Future Work}
This paper established that an agent trying to maximize a prediction reward naturally maximizes a lower bound on the information gain. This insight helps tie together multiple disparate sub-fields of machine learning that use prediction rewards and information gain separately. The DAN algorithm follows as a consequence of these results, which uses a model-free RL agent to gather data, based on prediction rewards, while simultaneously learning the predictions. We show that the approach improves performance in both a sensor selection and two visual attention tasks.

\section{Acknowledgement}
We would like to thank anonymous reviewers for their comments. This project has received funding from the European Research Council (ERC) under the European Union’s Horizon 2020 research and innovation programme (grant agreement number 637713). This project had received funding from the European Research Council (ERC) under the European Union's Horizon 2020 research and innovation programme (grant agreement No.~758824 \textemdash INFLUENCE). 
\hfill ~ \includegraphics[width=0.25\columnwidth]{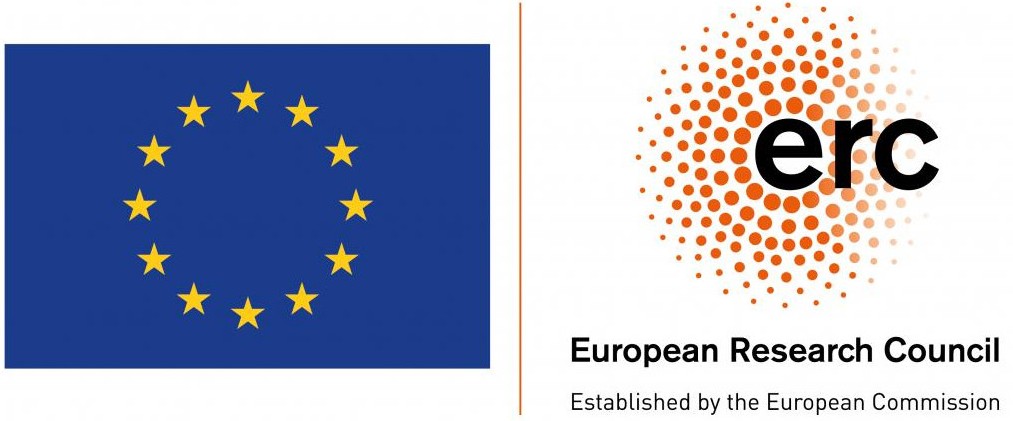}.

\bibliographystyle{ACM-Reference-Format}  
\bibliography{danBib}


\begin{thebibliography}{00}


\ifx \showCODEN    \undefined \def \showCODEN     #1{\unskip}     \fi
\ifx \showDOI      \undefined \def \showDOI       #1{#1}\fi
\ifx \showISBNx    \undefined \def \showISBNx     #1{\unskip}     \fi
\ifx \showISBNxiii \undefined \def \showISBNxiii  #1{\unskip}     \fi
\ifx \showISSN     \undefined \def \showISSN      #1{\unskip}     \fi
\ifx \showLCCN     \undefined \def \showLCCN      #1{\unskip}     \fi
\ifx \shownote     \undefined \def \shownote      #1{#1}          \fi
\ifx \showarticletitle \undefined \def \showarticletitle #1{#1}   \fi
\ifx \showURL      \undefined \def \showURL       {\relax}        \fi
\providecommand\bibfield[2]{#2}
\providecommand\bibinfo[2]{#2}
\providecommand\natexlab[1]{#1}
\providecommand\showeprint[2][]{arXiv:#2}

\bibitem[\protect\citeauthoryear{Allen}{Allen}{1985}]%
        {Allen85}
\bibfield{author}{\bibinfo{person}{P~K Allen}.}
  \bibinfo{year}{1985}\natexlab{}.
\newblock {\em \bibinfo{title}{Object recognition using vision and touch}}.
\newblock \bibinfo{thesistype}{Ph.D. Dissertation}. \bibinfo{school}{U of
  Penn.}
\newblock


\bibitem[\protect\citeauthoryear{Araya-l{\'o}pez, Thomas, Buffet, and
  Charpillet}{Araya-l{\'o}pez et~al\mbox{.}}{2010}]%
        {Araya10}
\bibfield{author}{\bibinfo{person}{M Araya-l{\'o}pez}, \bibinfo{person}{V
  Thomas}, \bibinfo{person}{O Buffet}, {and} \bibinfo{person}{F Charpillet}.}
  \bibinfo{year}{2010}\natexlab{}.
\newblock \showarticletitle{A {POMDP} extension with belief-dependent rewards}.
  In \bibinfo{booktitle}{{\em NeurIPS}}. \bibinfo{pages}{64--72}.
\newblock


\bibitem[\protect\citeauthoryear{Bachman, Sordoni, and Trischler}{Bachman
  et~al\mbox{.}}{2017}]%
        {Bachman17}
\bibfield{author}{\bibinfo{person}{P Bachman}, \bibinfo{person}{A Sordoni},
  {and} \bibinfo{person}{A Trischler}.} \bibinfo{year}{2017}\natexlab{}.
\newblock \showarticletitle{Learning algorithms for active learning}. In
  \bibinfo{booktitle}{{\em ICML}}. JMLR. org, \bibinfo{pages}{301--310}.
\newblock


\bibitem[\protect\citeauthoryear{Bajcsy}{Bajcsy}{1988}]%
        {Bajcsy88}
\bibfield{author}{\bibinfo{person}{R Bajcsy}.} \bibinfo{year}{1988}\natexlab{}.
\newblock \showarticletitle{Active perception}.
\newblock \bibinfo{journal}{{\it Proc. IEEE}} \bibinfo{volume}{76},
  \bibinfo{number}{8} (\bibinfo{year}{1988}), \bibinfo{pages}{966--1005}.
\newblock


\bibitem[\protect\citeauthoryear{Bajcsy, Aloimonos, and Tsotsos}{Bajcsy
  et~al\mbox{.}}{2018}]%
        {Bajcsy16}
\bibfield{author}{\bibinfo{person}{R Bajcsy}, \bibinfo{person}{Y Aloimonos},
  {and} \bibinfo{person}{J~K Tsotsos}.} \bibinfo{year}{2018}\natexlab{}.
\newblock \showarticletitle{Revisiting active perception}.
\newblock \bibinfo{journal}{{\em Autonomous Robots\/}} \bibinfo{volume}{42},
  \bibinfo{number}{2} (\bibinfo{year}{2018}), \bibinfo{pages}{177--196}.
\newblock


\bibitem[\protect\citeauthoryear{Bauschke and Lucet}{Bauschke and
  Lucet}{2012}]%
        {bauschke12}
\bibfield{author}{\bibinfo{person}{H Bauschke} {and} \bibinfo{person}{Y
  Lucet}.} \bibinfo{year}{2012}\natexlab{}.
\newblock \showarticletitle{What is a fenchel conjugate?}
\newblock \bibinfo{journal}{{\em Notices of the AMS\/}} (\bibinfo{year}{2012}),
  \bibinfo{pages}{44--46}.
\newblock


\bibitem[\protect\citeauthoryear{Belghazi, Baratin, Rajeswar, Ozair, Bengio,
  Courville, and Hjelm}{Belghazi et~al\mbox{.}}{2018}]%
        {Belghazi18}
\bibfield{author}{\bibinfo{person}{M~I Belghazi}, \bibinfo{person}{A Baratin},
  \bibinfo{person}{S Rajeswar}, \bibinfo{person}{S Ozair}, \bibinfo{person}{Y
  Bengio}, \bibinfo{person}{A Courville}, {and} \bibinfo{person}{R~D Hjelm}.}
  \bibinfo{year}{2018}\natexlab{}.
\newblock \showarticletitle{Mine: mutual information neural estimation}.
\newblock \bibinfo{journal}{{\em arXiv preprint arXiv:1801.04062\/}}
  (\bibinfo{year}{2018}), \bibinfo{pages}{2122--2131}.
\newblock


\bibitem[\protect\citeauthoryear{Bouma, Baan, Landsmeer, Kruszynski, van
  Antwerpen, and Dijk}{Bouma et~al\mbox{.}}{2013}]%
        {Bouma13}
\bibfield{author}{\bibinfo{person}{H Bouma}, \bibinfo{person}{J Baan},
  \bibinfo{person}{S Landsmeer}, \bibinfo{person}{C Kruszynski},
  \bibinfo{person}{G van Antwerpen}, {and} \bibinfo{person}{J Dijk}.}
  \bibinfo{year}{2013}\natexlab{}.
\newblock \showarticletitle{Real-time tracking and fast retrieval of persons in
  multiple surveillance cameras of a shopping mall}. In
  \bibinfo{booktitle}{{\em Multisensor, Multisource Information Fusion}},
  Vol.~\bibinfo{volume}{8756}. \bibinfo{pages}{87560A}.
\newblock


\bibitem[\protect\citeauthoryear{Boyd and Vandenberghe}{Boyd and
  Vandenberghe}{2004}]%
        {boyd04}
\bibfield{author}{\bibinfo{person}{S Boyd} {and} \bibinfo{person}{L
  Vandenberghe}.} \bibinfo{year}{2004}\natexlab{}.
\newblock \bibinfo{booktitle}{{\em Convex optimization}}.
\newblock \bibinfo{publisher}{Cambridge university press}.
\newblock


\bibitem[\protect\citeauthoryear{Bruce and Tsotsos}{Bruce and Tsotsos}{2009}]%
        {Bruce09}
\bibfield{author}{\bibinfo{person}{N~DB Bruce} {and} \bibinfo{person}{J~K
  Tsotsos}.} \bibinfo{year}{2009}\natexlab{}.
\newblock \showarticletitle{Saliency, attention, and visual search: An
  information theoretic approach}.
\newblock \bibinfo{journal}{{\em Journal of Vision\/}} \bibinfo{volume}{9},
  \bibinfo{number}{3} (\bibinfo{year}{2009}), \bibinfo{pages}{5--5}.
\newblock


\bibitem[\protect\citeauthoryear{Buck, Bulian, Ciaramita, Gajewski, Gesmundo,
  Houlsby, and Wang.}{Buck et~al\mbox{.}}{2018}]%
        {Buck18}
\bibfield{author}{\bibinfo{person}{C Buck}, \bibinfo{person}{J Bulian},
  \bibinfo{person}{M Ciaramita}, \bibinfo{person}{W Gajewski},
  \bibinfo{person}{A Gesmundo}, \bibinfo{person}{N Houlsby}, {and}
  \bibinfo{person}{W Wang.}} \bibinfo{year}{2018}\natexlab{}.
\newblock \showarticletitle{Ask the right questions: Active question
  reformulation with reinforcement learning}.
\newblock  (\bibinfo{year}{2018}), \bibinfo{pages}{1--15}.
\newblock


\bibitem[\protect\citeauthoryear{Burgard, Fox, and Thrun}{Burgard
  et~al\mbox{.}}{1997}]%
        {Burgard97}
\bibfield{author}{\bibinfo{person}{W Burgard}, \bibinfo{person}{D Fox}, {and}
  \bibinfo{person}{S Thrun}.} \bibinfo{year}{1997}\natexlab{}.
\newblock \showarticletitle{Active mobile robot localization by entropy
  minimization}. In \bibinfo{booktitle}{{\em EUROMICRO Workshop}}. IEEE,
  \bibinfo{pages}{155--162}.
\newblock


\bibitem[\protect\citeauthoryear{Doucet and Johansen}{Doucet and
  Johansen}{2009}]%
        {Doucet09}
\bibfield{author}{\bibinfo{person}{A Doucet} {and} \bibinfo{person}{A~M
  Johansen}.} \bibinfo{year}{2009}\natexlab{}.
\newblock \showarticletitle{A tutorial on particle filtering and smoothing:
  Fifteen years later}.
\newblock  (\bibinfo{year}{2009}), \bibinfo{pages}{656--704}.
\newblock


\bibitem[\protect\citeauthoryear{Eysenbach, Gupta, Ibarz, and Levine}{Eysenbach
  et~al\mbox{.}}{2018}]%
        {Eysenbach18}
\bibfield{author}{\bibinfo{person}{B Eysenbach}, \bibinfo{person}{A Gupta},
  \bibinfo{person}{J Ibarz}, {and} \bibinfo{person}{S Levine}.}
  \bibinfo{year}{2018}\natexlab{}.
\newblock \showarticletitle{Diversity is all you need: Learning skills without
  a reward function}.
\newblock \bibinfo{journal}{{\em arXiv preprint arXiv:1802.06070\/}}
  (\bibinfo{year}{2018}), \bibinfo{pages}{1--22}.
\newblock


\bibitem[\protect\citeauthoryear{Ghasemi and Topcu}{Ghasemi and Topcu}{2019}]%
        {Ghasemi19}
\bibfield{author}{\bibinfo{person}{M Ghasemi} {and} \bibinfo{person}{U Topcu}.}
  \bibinfo{year}{2019}\natexlab{}.
\newblock \showarticletitle{Online active perception for partially observable
  {M}arkov decision process with limited budget}.
\newblock \bibinfo{journal}{{\em arXiv preprint arXiv:1910.02130\/}}
  (\bibinfo{year}{2019}), \bibinfo{pages}{1--7}.
\newblock


\bibitem[\protect\citeauthoryear{Goodfellow, Pouget-Abadie, Mirza, Xu,
  Warde-Farley, Ozair, Courville, and Bengio}{Goodfellow et~al\mbox{.}}{2014}]%
        {Goodfellow14}
\bibfield{author}{\bibinfo{person}{I Goodfellow}, \bibinfo{person}{J
  Pouget-Abadie}, \bibinfo{person}{M Mirza}, \bibinfo{person}{B Xu},
  \bibinfo{person}{D Warde-Farley}, \bibinfo{person}{S Ozair},
  \bibinfo{person}{A Courville}, {and} \bibinfo{person}{Y Bengio}.}
  \bibinfo{year}{2014}\natexlab{}.
\newblock \showarticletitle{Generative adversarial nets}. In
  \bibinfo{booktitle}{{\em NeurIPS}}. \bibinfo{pages}{2672--2680}.
\newblock


\bibitem[\protect\citeauthoryear{Haque, Alahi, and Fei-Fei}{Haque
  et~al\mbox{.}}{2016}]%
        {Haque16}
\bibfield{author}{\bibinfo{person}{A Haque}, \bibinfo{person}{A Alahi}, {and}
  \bibinfo{person}{L Fei-Fei}.} \bibinfo{year}{2016}\natexlab{}.
\newblock \showarticletitle{Recurrent attention models for depth-based person
  identification}. In \bibinfo{booktitle}{{\em Proceedings of the IEEE
  Conference on Computer Vision and Pattern Recognition}}.
  \bibinfo{pages}{1229--1238}.
\newblock


\bibitem[\protect\citeauthoryear{Hausknecht and Stone}{Hausknecht and
  Stone}{2015}]%
        {Hausknecht15}
\bibfield{author}{\bibinfo{person}{M. Hausknecht} {and} \bibinfo{person}{P
  Stone}.} \bibinfo{year}{2015}\natexlab{}.
\newblock \showarticletitle{Deep recurrent {Q}-learning for partially
  observable {MDP}s}. In \bibinfo{booktitle}{{\em 2015 AAAI Fall Symposium
  Series}}. \bibinfo{pages}{29--37}.
\newblock


\bibitem[\protect\citeauthoryear{Hero and Cochran}{Hero and Cochran}{2011}]%
        {Hero11}
\bibfield{author}{\bibinfo{person}{A~O Hero} {and} \bibinfo{person}{D
  Cochran}.} \bibinfo{year}{2011}\natexlab{}.
\newblock \showarticletitle{Sensor management: Past, present, and future}.
\newblock \bibinfo{journal}{{\em IEEE Sensors Journal\/}} \bibinfo{volume}{11},
  \bibinfo{number}{12} (\bibinfo{year}{2011}), \bibinfo{pages}{3064--3075}.
\newblock


\bibitem[\protect\citeauthoryear{Igl, Zintgraf, Le, Wood, and Whiteson}{Igl
  et~al\mbox{.}}{2018}]%
        {Igl18}
\bibfield{author}{\bibinfo{person}{M Igl}, \bibinfo{person}{L Zintgraf},
  \bibinfo{person}{T~A Le}, \bibinfo{person}{F Wood}, {and} \bibinfo{person}{S
  Whiteson}.} \bibinfo{year}{2018}\natexlab{}.
\newblock \showarticletitle{Deep variational reinforcement learning for
  {POMDP}s}. In \bibinfo{booktitle}{{\em ICML}}. \bibinfo{pages}{2117--2126}.
\newblock


\bibitem[\protect\citeauthoryear{Jaderberg, Mnih, Czarnecki, Schaul, Leibo,
  Silver, and Kavukcuoglu}{Jaderberg et~al\mbox{.}}{2016}]%
        {Jaderberg16}
\bibfield{author}{\bibinfo{person}{M Jaderberg}, \bibinfo{person}{V Mnih},
  \bibinfo{person}{W~M Czarnecki}, \bibinfo{person}{T Schaul},
  \bibinfo{person}{J~Z Leibo}, \bibinfo{person}{D Silver}, {and}
  \bibinfo{person}{K Kavukcuoglu}.} \bibinfo{year}{2016}\natexlab{}.
\newblock \showarticletitle{Reinforcement learning with unsupervised auxiliary
  tasks}. In \bibinfo{booktitle}{{\em ICLR}}. \bibinfo{pages}{1--17}.
\newblock


\bibitem[\protect\citeauthoryear{James and Singh}{James and Singh}{2009}]%
        {James09}
\bibfield{author}{\bibinfo{person}{M~R James} {and} \bibinfo{person}{S Singh}.}
  \bibinfo{year}{2009}\natexlab{}.
\newblock \showarticletitle{SarsaLandmark: an algorithm for learning in
  {POMDP}s with landmarks}. In \bibinfo{booktitle}{{\em AAMAS}}. International
  Foundation for Autonomous Agents and Multiagent Systems,
  \bibinfo{pages}{585--591}.
\newblock


\bibitem[\protect\citeauthoryear{Joshi and Boyd}{Joshi and Boyd}{2009}]%
        {Joshi09}
\bibfield{author}{\bibinfo{person}{S Joshi} {and} \bibinfo{person}{S Boyd}.}
  \bibinfo{year}{2009}\natexlab{}.
\newblock \showarticletitle{Sensor selection via convex optimization}.
\newblock \bibinfo{journal}{{\em IEEE TSP\/}} (\bibinfo{year}{2009}),
  \bibinfo{pages}{451--462}.
\newblock


\bibitem[\protect\citeauthoryear{Kaelbling, Littman, and Cassandra}{Kaelbling
  et~al\mbox{.}}{1998}]%
        {Kaelbling98}
\bibfield{author}{\bibinfo{person}{L~P Kaelbling}, \bibinfo{person}{M~L.
  Littman}, {and} \bibinfo{person}{A~R Cassandra}.}
  \bibinfo{year}{1998}\natexlab{}.
\newblock \showarticletitle{Planning and acting in partially observable
  stochastic domains}.
\newblock \bibinfo{journal}{{\em Artificial Intelligence\/}}
  (\bibinfo{year}{1998}), \bibinfo{pages}{99--134}.
\newblock


\bibitem[\protect\citeauthoryear{Katt, Oliehoek, and Amato}{Katt
  et~al\mbox{.}}{2017}]%
        {Katt17}
\bibfield{author}{\bibinfo{person}{S Katt}, \bibinfo{person}{F~A Oliehoek},
  {and} \bibinfo{person}{C Amato}.} \bibinfo{year}{2017}\natexlab{}.
\newblock \showarticletitle{Learning in {POMDP}s with {M}onte {C}arlo Tree
  Search}. In \bibinfo{booktitle}{{\em ICML}} {\em
  (\bibinfo{series}{Proceedings of Machine Learning Research})},
  Vol.~\bibinfo{volume}{70}. \bibinfo{publisher}{PMLR},
  \bibinfo{pages}{1819--1827}.
\newblock


\bibitem[\protect\citeauthoryear{Kelly}{Kelly}{1971}]%
        {Kelly70}
\bibfield{author}{\bibinfo{person}{M~D Kelly}.}
  \bibinfo{year}{1971}\natexlab{}.
\newblock \showarticletitle{Edge detection in pictures by computer using
  planning}.
\newblock \bibinfo{journal}{{\em Machine Intelligence\/}}
  (\bibinfo{year}{1971}), \bibinfo{pages}{397--409}.
\newblock


\bibitem[\protect\citeauthoryear{Kingma and Ba}{Kingma and Ba}{2014}]%
        {Kingma14}
\bibfield{author}{\bibinfo{person}{D Kingma} {and} \bibinfo{person}{J Ba}.}
  \bibinfo{year}{2014}\natexlab{}.
\newblock \showarticletitle{Adam: A method for stochastic optimization}.
\newblock \bibinfo{journal}{{\em ICLR\/}}, \bibinfo{pages}{1--15}.
\newblock


\bibitem[\protect\citeauthoryear{Kostrikov, Erhan, and Levine}{Kostrikov
  et~al\mbox{.}}{2016}]%
        {Kostrikov16}
\bibfield{author}{\bibinfo{person}{I Kostrikov}, \bibinfo{person}{. Erhan},
  {and} \bibinfo{person}{S Levine}.} \bibinfo{year}{2016}\natexlab{}.
\newblock \showarticletitle{End to end active perception}. In
  \bibinfo{booktitle}{{\em NIPS 2016 Deep Learning Symposium}}.
  \bibinfo{pages}{1--9}.
\newblock


\bibitem[\protect\citeauthoryear{Krause and Guestrin}{Krause and
  Guestrin}{2005}]%
        {Krause05uai}
\bibfield{author}{\bibinfo{person}{A Krause} {and} \bibinfo{person}{C
  Guestrin}.} \bibinfo{year}{2005}\natexlab{}.
\newblock \showarticletitle{Near-optimal nonmyopic value of information in
  graphical models}. In \bibinfo{booktitle}{{\em UAI}}.
  \bibinfo{pages}{324--331}.
\newblock


\bibitem[\protect\citeauthoryear{Kreucher, Kastella, and Hero}{Kreucher
  et~al\mbox{.}}{2005}]%
        {Kreucher05}
\bibfield{author}{\bibinfo{person}{C Kreucher}, \bibinfo{person}{K Kastella},
  {and} \bibinfo{person}{A~O Hero}.} \bibinfo{year}{2005}\natexlab{}.
\newblock \showarticletitle{Sensor management using an active sensing
  approach}.
\newblock \bibinfo{journal}{{\em Signal Processing\/}} \bibinfo{volume}{85},
  \bibinfo{number}{3} (\bibinfo{year}{2005}), \bibinfo{pages}{607--624}.
\newblock


\bibitem[\protect\citeauthoryear{Le, Saxena, and Ng}{Le et~al\mbox{.}}{2008}]%
        {Le08}
\bibfield{author}{\bibinfo{person}{Q~V Le}, \bibinfo{person}{A Saxena}, {and}
  \bibinfo{person}{A~Y Ng}.} \bibinfo{year}{2008}\natexlab{}.
\newblock \showarticletitle{Active perception: Interactive manipulation for
  improving object detection}.
\newblock \bibinfo{journal}{{\em Stanford University Journal\/}}
  (\bibinfo{year}{2008}), \bibinfo{pages}{1--9}.
\newblock


\bibitem[\protect\citeauthoryear{LeCun, Bottou, Bengio, and Haffner}{LeCun
  et~al\mbox{.}}{1998}]%
        {Lecun98}
\bibfield{author}{\bibinfo{person}{Y LeCun}, \bibinfo{person}{L Bottou},
  \bibinfo{person}{Y Bengio}, {and} \bibinfo{person}{P Haffner}.}
  \bibinfo{year}{1998}\natexlab{}.
\newblock \showarticletitle{Gradient-based learning applied to document
  recognition}.
\newblock \bibinfo{journal}{{\it Proc. IEEE}} \bibinfo{volume}{86},
  \bibinfo{number}{11} (\bibinfo{year}{1998}), \bibinfo{pages}{2278--2324}.
\newblock


\bibitem[\protect\citeauthoryear{Mnih, Heess, Graves, and Kavukcuoglu}{Mnih
  et~al\mbox{.}}{2014}]%
        {Mnih14}
\bibfield{author}{\bibinfo{person}{V Mnih}, \bibinfo{person}{N Heess},
  \bibinfo{person}{A Graves}, {and} \bibinfo{person}{K Kavukcuoglu}.}
  \bibinfo{year}{2014}\natexlab{}.
\newblock \showarticletitle{Recurrent models of visual attention}. In
  \bibinfo{booktitle}{{\em NeurIPS}}. \bibinfo{pages}{2204--2212}.
\newblock


\bibitem[\protect\citeauthoryear{Mohamed and Rezende}{Mohamed and
  Rezende}{2015}]%
        {Mohamed15}
\bibfield{author}{\bibinfo{person}{S Mohamed} {and} \bibinfo{person}{D~J
  Rezende}.} \bibinfo{year}{2015}\natexlab{}.
\newblock \showarticletitle{Variational information maximisation for
  intrinsically motivated reinforcement learning}. In \bibinfo{booktitle}{{\em
  Neurips}}. \bibinfo{pages}{2125--2133}.
\newblock


\bibitem[\protect\citeauthoryear{Monari and Kroschel}{Monari and
  Kroschel}{2010}]%
        {Monari10}
\bibfield{author}{\bibinfo{person}{E Monari} {and} \bibinfo{person}{K
  Kroschel}.} \bibinfo{year}{2010}\natexlab{}.
\newblock \showarticletitle{Dynamic sensor selection for single target tracking
  in large video surveillance networks}. In \bibinfo{booktitle}{{\em IEEE
  AVSS}}. IEEE, \bibinfo{pages}{539--546}.
\newblock


\bibitem[\protect\citeauthoryear{Mousavi, Liu, Yuan, Takáč, Muñoz-Avila, and
  Motee}{Mousavi et~al\mbox{.}}{2019}]%
        {Mousavi19}
\bibfield{author}{\bibinfo{person}{H~K Mousavi}, \bibinfo{person}{G Liu},
  \bibinfo{person}{W Yuan}, \bibinfo{person}{M Takáč}, \bibinfo{person}{H
  Muñoz-Avila}, {and} \bibinfo{person}{N Motee}.}
  \bibinfo{year}{2019}\natexlab{}.
\newblock \showarticletitle{A layered architecture for active perception: Image
  classification using deep reinforcement learning}.
\newblock \bibinfo{journal}{{\em arXiv preprint arXiv:1909.09705\/}}
  (\bibinfo{year}{2019}), \bibinfo{pages}{1--7}.
\newblock


\bibitem[\protect\citeauthoryear{Narasimhan, Yala, and Barzilay}{Narasimhan
  et~al\mbox{.}}{2016}]%
        {Nara16}
\bibfield{author}{\bibinfo{person}{K Narasimhan}, \bibinfo{person}{A Yala},
  {and} \bibinfo{person}{R Barzilay}.} \bibinfo{year}{2016}\natexlab{}.
\newblock \showarticletitle{Improving information extraction by acquiring
  external evidence with reinforcement learning}. In \bibinfo{booktitle}{{\em
  EMNLP}}. \bibinfo{pages}{2355--2365}.
\newblock


\bibitem[\protect\citeauthoryear{Nowozin}{Nowozin}{2012}]%
        {Nowozin12}
\bibfield{author}{\bibinfo{person}{S Nowozin}.}
  \bibinfo{year}{2012}\natexlab{}.
\newblock \showarticletitle{Improved information gain estimates for decision
  tree induction}. In \bibinfo{booktitle}{{\em ICML}}.
  \bibinfo{pages}{571--578}.
\newblock


\bibitem[\protect\citeauthoryear{Oh, Chockalingam, Singh, and Lee}{Oh
  et~al\mbox{.}}{2016}]%
        {Oh16}
\bibfield{author}{\bibinfo{person}{J Oh}, \bibinfo{person}{V Chockalingam},
  \bibinfo{person}{S Singh}, {and} \bibinfo{person}{H Lee}.}
  \bibinfo{year}{2016}\natexlab{}.
\newblock \showarticletitle{Control of memory, active perception, and action in
  minecraft}. In \bibinfo{booktitle}{{\em ICML}}. \bibinfo{pages}{2790--2799}.
\newblock


\bibitem[\protect\citeauthoryear{Pathak, Agrawal, Efros, and Darrell}{Pathak
  et~al\mbox{.}}{2017}]%
        {Pathak17}
\bibfield{author}{\bibinfo{person}{D Pathak}, \bibinfo{person}{P Agrawal},
  \bibinfo{person}{A~A Efros}, {and} \bibinfo{person}{T Darrell}.}
  \bibinfo{year}{2017}\natexlab{}.
\newblock \showarticletitle{Curiosity-driven exploration by self-supervised
  prediction}. In \bibinfo{booktitle}{{\em ICML}}. \bibinfo{pages}{2778--2787}.
\newblock


\bibitem[\protect\citeauthoryear{Satsangi, Whiteson, Oliehoek, and
  Spaan}{Satsangi et~al\mbox{.}}{2018}]%
        {SatsangiJournal17}
\bibfield{author}{\bibinfo{person}{Y Satsangi}, \bibinfo{person}{S Whiteson},
  \bibinfo{person}{F~A. Oliehoek}, {and} \bibinfo{person}{M Spaan}.}
  \bibinfo{year}{2018}\natexlab{}.
\newblock \showarticletitle{Exploiting submodularity for scaling Up active
  perception}.
\newblock \bibinfo{journal}{{\em Autonomous Robots\/}} \bibinfo{volume}{42},
  \bibinfo{number}{2} (\bibinfo{year}{2018}), \bibinfo{pages}{209--233}.
\newblock


\bibitem[\protect\citeauthoryear{Schulman, Wolski, Dhariwal, Radford, and
  Klimov}{Schulman et~al\mbox{.}}{2017}]%
        {Schulman17}
\bibfield{author}{\bibinfo{person}{J Schulman}, \bibinfo{person}{F Wolski},
  \bibinfo{person}{P Dhariwal}, \bibinfo{person}{A Radford}, {and}
  \bibinfo{person}{O Klimov}.} \bibinfo{year}{2017}\natexlab{}.
\newblock \showarticletitle{Proximal policy optimization algorithms}.
\newblock \bibinfo{journal}{{\em arXiv preprint arXiv:1707.06347\/}}
  (\bibinfo{year}{2017}), \bibinfo{pages}{1--12}.
\newblock


\bibitem[\protect\citeauthoryear{Spaan and Lima}{Spaan and Lima}{2009}]%
        {Spaan09}
\bibfield{author}{\bibinfo{person}{M~T~J Spaan} {and} \bibinfo{person}{P~U
  Lima}.} \bibinfo{year}{2009}\natexlab{}.
\newblock \showarticletitle{A decision-theoretic approach to dynamic sensor
  selection in camera networks}. In \bibinfo{booktitle}{{\em ICAPS}}.
  \bibinfo{pages}{279--304}.
\newblock


\bibitem[\protect\citeauthoryear{Spaan, Veiga, and Lima}{Spaan
  et~al\mbox{.}}{2015}]%
        {Spaan15}
\bibfield{author}{\bibinfo{person}{M~T~J Spaan}, \bibinfo{person}{T~S Veiga},
  {and} \bibinfo{person}{P~U Lima}.} \bibinfo{year}{2015}\natexlab{}.
\newblock \showarticletitle{Decision-theoretic planning under uncertainty with
  information rewards for active cooperative perception}.
\newblock \bibinfo{journal}{{\em AAMAS\/}} \bibinfo{volume}{29},
  \bibinfo{number}{6} (\bibinfo{year}{2015}), \bibinfo{pages}{1157--1185}.
\newblock


\bibitem[\protect\citeauthoryear{Srivastava, Hinton, Krizhevsky, Sutskever, and
  Salakhutdinov}{Srivastava et~al\mbox{.}}{2014}]%
        {Srivastava14}
\bibfield{author}{\bibinfo{person}{N Srivastava}, \bibinfo{person}{G Hinton},
  \bibinfo{person}{A Krizhevsky}, \bibinfo{person}{I Sutskever}, {and}
  \bibinfo{person}{R Salakhutdinov}.} \bibinfo{year}{2014}\natexlab{}.
\newblock \showarticletitle{Dropout: a simple way to prevent neural networks
  from overfitting}.
\newblock \bibinfo{journal}{{\em JMLR\/}} \bibinfo{volume}{15},
  \bibinfo{number}{1} (\bibinfo{year}{2014}), \bibinfo{pages}{1929--1958}.
\newblock


\bibitem[\protect\citeauthoryear{Tessens, Morbee, Aghajan, and Philips}{Tessens
  et~al\mbox{.}}{2014}]%
        {Tessens14}
\bibfield{author}{\bibinfo{person}{L Tessens}, \bibinfo{person}{M Morbee},
  \bibinfo{person}{H Aghajan}, {and} \bibinfo{person}{W Philips}.}
  \bibinfo{year}{2014}\natexlab{}.
\newblock \showarticletitle{Camera selection for tracking in distributed smart
  camera networks}.
\newblock \bibinfo{journal}{{\em ACM TOSN\/}} \bibinfo{volume}{10},
  \bibinfo{number}{2} (\bibinfo{year}{2014}), \bibinfo{pages}{23}.
\newblock


\bibitem[\protect\citeauthoryear{Van~Hasselt, Guez, and Silver}{Van~Hasselt
  et~al\mbox{.}}{2016}]%
        {Hasselt16}
\bibfield{author}{\bibinfo{person}{H Van~Hasselt}, \bibinfo{person}{A Guez},
  {and} \bibinfo{person}{D Silver}.} \bibinfo{year}{2016}\natexlab{}.
\newblock \showarticletitle{Deep reinforcement learning with double
  q-learning}. In \bibinfo{booktitle}{{\em Thirtieth AAAI conference on
  artificial intelligence}}. \bibinfo{pages}{2094--2100}.
\newblock


\bibitem[\protect\citeauthoryear{Wilkes and Tsotsos}{Wilkes and
  Tsotsos}{1992}]%
        {Wilkes92}
\bibfield{author}{\bibinfo{person}{D Wilkes} {and} \bibinfo{person}{J~K
  Tsotsos}.} \bibinfo{year}{1992}\natexlab{}.
\newblock \showarticletitle{Active object recognition}. In
  \bibinfo{booktitle}{{\em CVPR}}. IEEE, \bibinfo{pages}{136--141}.
\newblock


\bibitem[\protect\citeauthoryear{Williams, Fisher, and Willsky}{Williams
  et~al\mbox{.}}{2007}]%
        {Williams07}
\bibfield{author}{\bibinfo{person}{J~L Williams}, \bibinfo{person}{J~W Fisher},
  {and} \bibinfo{person}{A~S Willsky}.} \bibinfo{year}{2007}\natexlab{}.
\newblock \showarticletitle{Approximate dynamic programming for
  communication-constrained sensor network management}.
\newblock \bibinfo{journal}{{\em IEEE TSP\/}} \bibinfo{volume}{55},
  \bibinfo{number}{8} (\bibinfo{year}{2007}), \bibinfo{pages}{4300--4311}.
\newblock


\bibitem[\protect\citeauthoryear{Williams}{Williams}{1992}]%
        {Williams92}
\bibfield{author}{\bibinfo{person}{R~J Williams}.}
  \bibinfo{year}{1992}\natexlab{}.
\newblock \showarticletitle{Simple statistical gradient-following algorithms
  for connectionist reinforcement learning}.
\newblock \bibinfo{journal}{{\em Machine learning\/}} \bibinfo{volume}{8},
  \bibinfo{number}{3-4} (\bibinfo{year}{1992}), \bibinfo{pages}{229--256}.
\newblock


\bibitem[\protect\citeauthoryear{Xiao, Rasul, and Vollgraf}{Xiao
  et~al\mbox{.}}{2017}]%
        {Xiao17}
\bibfield{author}{\bibinfo{person}{H Xiao}, \bibinfo{person}{K Rasul}, {and}
  \bibinfo{person}{R Vollgraf}.} \bibinfo{year}{2017}\natexlab{}.
\newblock \showarticletitle{Fashion-{MNIST}: a novel image dataset for
  benchmarking machine learning algorithms}.
\newblock \bibinfo{journal}{{\em arXiv preprint arXiv:1708.07747\/}}
  (\bibinfo{year}{2017}).
\newblock


\bibitem[\protect\citeauthoryear{Yang, Wolpert, and Lengyel}{Yang
  et~al\mbox{.}}{2016}]%
        {Yang16}
\bibfield{author}{\bibinfo{person}{S~C~H Yang}, \bibinfo{person}{D~M Wolpert},
  {and} \bibinfo{person}{M Lengyel}.} \bibinfo{year}{2016}\natexlab{}.
\newblock \showarticletitle{Theoretical perspectives on active sensing}.
\newblock \bibinfo{journal}{{\em Current opinion in behavioral sciences\/}}
  \bibinfo{volume}{11} (\bibinfo{year}{2016}), \bibinfo{pages}{100--108}.
\newblock


\bibitem[\protect\citeauthoryear{Ye and Tsotsos}{Ye and Tsotsos}{1995}]%
        {Ye95}
\bibfield{author}{\bibinfo{person}{Y Ye} {and} \bibinfo{person}{J~K Tsotsos}.}
  \bibinfo{year}{1995}\natexlab{}.
\newblock \showarticletitle{Where to look next in 3d object search}. In
  \bibinfo{booktitle}{{\em ISCV}}. \bibinfo{publisher}{IEEE},
  \bibinfo{pages}{539--544}.
\newblock


\end{thebibliography}

\end{document}